\documentclass[10pt,twocolumn,twoside,journal]{IEEEtran}
\usepackage{algorithm,algorithmic,cite,amsmath,amssymb,epsfig,enumerate,amsfonts,url,multirow,array,tabularx,booktabs,stfloats,flushend}
\usepackage[T1]{fontenc}
\newcommand{\bm}[1]{\mbox{\boldmath{$#1$}}}
\renewcommand\arraystretch{1.3}
\renewcommand{\algorithmicrequire}{\textbf{Input:}}

\def\x{{\mathbf x}}
\def\w{{\mathbf w}}

\def\u{{\mathbf u}}
\def\z{{\mathbf z}}
\def\c{{\mathbf c}}
\def\g{{\mathbf g}}
\def\m{{\mathbf m}}

\def\M{{\mathbf M}}
\def\W{{\mathbf W}}
\def\G{{\mathbf G}}

\def\O{{\mathcal O}}
\def\J{{\mathcal J}}

\newtheorem{thm}{Theorem~}

\ifCLASSINFOpdf

\else

\fi

\begin{document}

\title{Heuristic Ternary Error-Correcting Output Codes Via Weight Optimization and Layered Clustering-Based Approach}

\author{~Xiao-Lei~Zhang,~\IEEEmembership{Member,~IEEE}

%\thanks{Manuscript received April 19, 20XX; revised January 11, 20XX.}
\thanks{This work was supported by the China Postdoctoral Science Foundation funded project under Grant 2012M520278.

The author was with the
Tsinghua National Laboratory for Information Science and Technology,
Department of Electronic Engineering, Tsinghua University, Beijing, China (e-mail: huoshan6@126.com).}
\thanks{Digital Object Identifier}
}

% The paper headers
%\markboth{IEEE TRANSACTIONS ON CYBERNETICS,~Vol.~X, No.~X, XXXX~20XX}%
%{Shell \MakeLowercase{\textit{et al.}}: Bare Demo of IEEEtran.cls for Journals}

\maketitle

\begin{abstract}
%\boldmath
One important classifier ensemble for multiclass classification problems is Error-Correcting Output Codes (ECOCs). It bridges multiclass problems and binary-class classifiers by decomposing multiclass problems to a serial binary-class problems. In this paper, we present a heuristic ternary code, named Weight Optimization and Layered Clustering-based ECOC (WOLC-ECOC). It starts with an arbitrary valid ECOC and iterates the following two steps until the training risk converges. The first step, named Layered Clustering based ECOC (LC-ECOC), constructs multiple strong classifiers on the most confusing binary-class problem. The second step adds the new classifiers to ECOC by a novel Optimized Weighted (OW) decoding algorithm, where the optimization problem of the decoding is solved by the cutting plane algorithm. Technically, LC-ECOC makes the heuristic training process not blocked by some difficult binary-class problem. OW decoding guarantees the non-increase of the training risk for ensuring a small code length. Results on 14 UCI datasets and a music genre classification problem demonstrate the effectiveness of WOLC-ECOC.

\end{abstract}
\begin{IEEEkeywords}
 Error-Correcting Output Code (ECOC), ensemble learning, multiple classifier system, multiclass classification.
\end{IEEEkeywords}

\IEEEpeerreviewmaketitle

 \setlength{\arraycolsep}{0.0em}

\section{Introduction}\label{sec:introduction}
\IEEEPARstart{O}{ver} the last decades, classifier ensembles \cite{dietterich2000ensemble,polikar2006ensemble,rahman2011novel,re2011ensemble,leung2007generating,xu2011efficient}, such as {\it{bagging}}\cite{breiman1996bagging}, {\it{boosting}}\cite{schapire1990strength}, and their variations, have demonstrated their effectiveness on many learning problems \cite{huang2012extreme,wang2012multiclass,zhang2012linearithmic}. Their success relies on a good selection of base learners and a strong {\it{diversity}} among the base learners, where the word ``diversity'' means that when the base learners make predictions on an identical pattern, they are different from each other in terms of errors.
As summarized in \cite{dietterich2000ensemble,polikar2006ensemble,rahman2011novel,re2011ensemble}, there are generally four groups of classifier ensembles: (i) manipulating training examples, (ii) manipulating input features, (iii) manipulating training parameters, and (iv) manipulating output targets.

One method of manipulating output targets is Error-Correcting Output Codes (ECOCs) \cite{dietterich1995solving}, which is motivated from information theory for correcting bits caused by noisy communication channels.
The key idea of ECOC is summarized as follows. Given a multiclass problem, ECOC assigns each class a unique codeword. All codewords form an ECOC {\it{coding matrix}}, where each row of the coding matrix is a \textit{codeword} and each column defines a bipartition of the classes. Training \textit{dichotomizers} (i.e., binary-class classifiers) on different bipartitions of the classes gets an ECOC ensemble. ECOC has two merits: (i) it bridges multiclass problems and dichotomizers, and (ii) it may correct errors by proper codeword designs.
ECOC consists of two parts---\textit{coding} and \textit{decoding}. Coding assigns each class a unique codeword. Decoding predicts a test pattern by matching the predicted codeword with its most similar codeword in the coding matrix.

The coding techniques can be categorized to two classes. The first class is {\it{problem-independent}} codings \cite{dietterich1995solving,tapia2004beyond,tapia2010recursive}, which use coding matrices that have strong error-correcting abilities in the view of channel coding. The second class is {\it{problem-dependent}} codings \cite{fung2005multicategory,weston1999support,guermeur2002combining,lee2004multicategory,chen2010multiclass,ghorai2010discriminant,zhong2010learning,pujol2006discriminant,pujol2008incremental,escalera2008subclass,bouzas2011optimizing,kuncheva2003measures,kuncheva2005using,escalera2009separability,escalera2009recoding,prior2005over,escalera2007boosted,bautista2010compact,bautista2010compact2}, which aim to solve given multiclass problems without considering the error-correcting ability of coding matrices much. This class attracted much attention in recent years, such as Discriminant ECOC (DECOC) \cite{pujol2006discriminant}, ECOC-Optimizing Node Embedding (ECOC-ONE) \cite{escalera2006ecoc,pujol2008incremental}, subclass-ECOC \cite{escalera2008subclass}, manipulations of features \cite{bagheri2012subspace,bagheri2012rough}, and manipulations of the parameters of base dichotomizers \cite{prior2005over}. The decoding methods are various distance metrics, including hamming Distance (HD), euclidean Distance (ED), probabilistic \cite{passerini2004new}, Loss Based (LB) \cite{allwein2001reducing}, and Loss Weighted (LW) \cite{escalera2010decoding} decodings.

In this paper, we propose a heuristic ternary ECOC, named Weight Optimization and Layered Clustering based ECOC (WOLC-ECOC). As shown in Fig. \ref{fig:fig_visio}, it begins with an arbitrary \textit{valid} ECOC ensemble and iteratively adds new dichotomizers to the ensemble in a greedy training manner by the following two steps until the training risk converges, where the word ``valid'' means that each codeword is unique. The first step trains a dichotomizer that discriminates the most confusing pair of classes by a new Layered Clustering-based ECOC (LC-ECOC) approach. The second step adds the dichotomizer to ECOC by a new Optimized Weighted (OW) decoding algorithm. The left side of the dotted line of Fig. \ref{fig:fig_visio} summarizes the contributions of this paper, while the right side was proposed in \cite{escalera2006ecoc,pujol2008incremental}:
 \begin{figure}[!t]
\centerline{
\includegraphics[width=80mm]{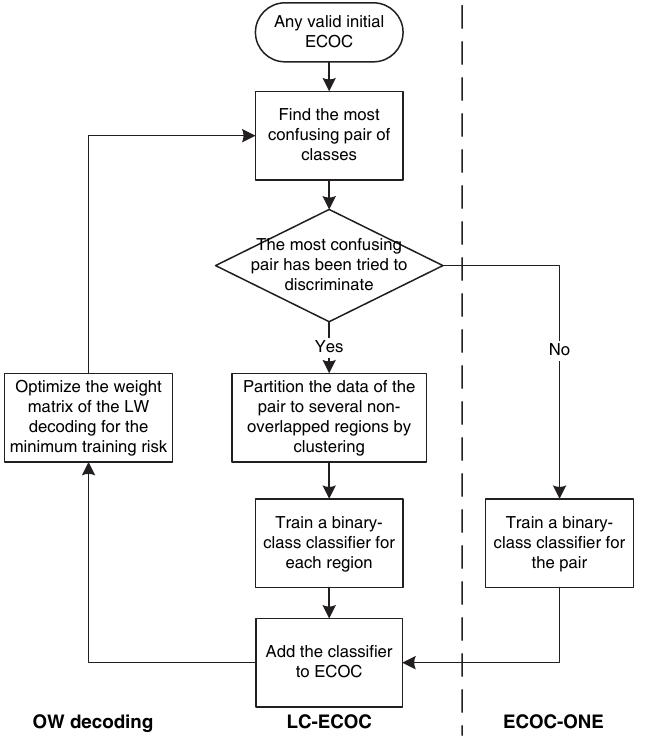}}
\caption{{System overview of WOLC-ECOC. The left side of the dotted line is our contribution. The right side of the dotted line is ECOC-ONE \cite{escalera2006ecoc,pujol2008incremental}.}}
\label{fig:fig_visio}
\end{figure}
\begin{itemize}
  \item {A novel LC-ECOC coding method is proposed.}
      The key idea of LC-ECOC is to construct multiple strong dichotomizers on a single pair of classes by first clustering the pair to small non-overlapped regions multiple times and then training a classifier for each region in each time of clustering, where all classifiers in each time of clustering group to a strong dichotomizer. It is motivated from the weakness of ECOC-ONE \cite{escalera2006ecoc,pujol2008incremental} in which the heuristic training process might be blocked by some difficult binary-class problems; although subclass-ECOC \cite{escalera2008subclass} has shown its advantage on the most confusing problems by embedding a tree into each problem, it is difficult to control the growth of the tree.
  \item {A novel Cutting-Plane Algorithm (CPA) based OW decoding method is proposed.} Like LW decoding \cite{escalera2010decoding}, OW decoding is also a non-biased decoding for ternary codes, but OW decoding improves LW decoding by optimizing the empirical weight matrix of the LW decoding for the minimum training risk.
       We solve the optimization problem via CPA \cite{kelley1960cutting,joachims2006training,teo2007scalable,franc2008optimized}. The CPA based OW decoding has linear time and storage complexities.
  \item {A novel WOLC-ECOC classifier system is proposed.} As shown in Fig. \ref{fig:fig_visio}, WOLC-ECOC iterates LC-ECOC (and also ECOC-ONE) and OW decoding until the training risk converges. The iteration integrates the merits of the aforementioned two items together: (i) LC-ECOC ensures that the greedy training will not be blocked by some difficult binary-class problems; (ii) OW decoding guarantees the non-increase of the training risk whenever adding a new dichotomizer to ECOC, so that the heuristic training can be easily controlled via the training risk, which makes a small code length available.

       WOLC-ECOC inherits the advantages of ECOC-ONE \cite{pujol2008incremental}, subclass-ECOC \cite{escalera2008subclass} and LW decoding \cite{escalera2010decoding}, and meanwhile overcomes their drawbacks.

  \item {A brief literature survey of ECOC is conducted.}
\end{itemize}

 The experimental comparison with 15 coding-decoding methods on 14 UCI benchmark datasets with 2 kinds of base classifiers shows that WOLC-ECOC outperforms comparison methods when the discrete Adaboost is used as the base classifier, outperforms 12 comparison methods when the Gaussian Radial-Basis-Function (RBF) kernel based SVM is used as the base classifier, and meanwhile maintains a small code length in both scenarios.

The rest of the paper is organized as follows. In Section \ref{sec:mmc}, we conduct a brief literature survey on ECOC. In Section \ref{sec:LC-ECOC}, we present the LC-ECOC coding method. In Section \ref{sec:calcweights}, we present the CPA based OW decoding method. In Section \ref{sec:mmc_dual}, we present WOLC-ECOC. In Section \ref{sec:experiments}, we report the experimental results and further apply WOLC-ECOC to a real-world problem---music genre classification. Finally, we conclude this paper in Section \ref{sec:conclusion}.

We first introduce some notations here. Bold small letters, e.g. $\w$, indicate column vectors. Bold capital letters, e.g. $\M$ and $\W$, indicate matrices. Letters in calligraphic fonts, e.g. $\mathcal{W}$, indicate sets, where $\mathbb{R}^d$ denotes a $d$-dimensional real space. $\mathbf{0}$ ($\mathbf{1}$) is a column vector with all entries being 1 (0).

\section{A Brief Literature Survey}\label{sec:mmc}

 ECOC originally views ``machine learning as a kind of communication problem in which the identity of the correct output class for a new example is being transmitted over a channel. The channel consists of the input features, the training examples, and the learning algorithm.'' \cite{dietterich1995solving}. Given a $P$ class classification problem with a set of labeled examples  $\{(\bm\rho_i,y_i)\}_{i=1}^n$ where $\bm\rho_i\in\mathbb{R}^d$ and $y_i\in\{1,2,\ldots,P \}$ is the label of $\bm\rho_i$, ECOC aims to solve the problem by for example $Q$ dichotomizers. The relation between the classes and the dichotomizers can be expressed by a \textit{binary} coding matrix $\M\in\{-1,1 \}^{P\times Q}$ or a \textit{ternary} coding matrix $\M\in\{-1,0,1 \}^{P\times Q}$, where the $p$-th row of $\M$ expresses the codeword of class $p$, denoted as $\c_p$, and the $q$-th column expresses the $q$-th dichotomizers, denoted as $h_q$.

\subsection{Survey on the Coding Phase}

 Two common output codes are the one-versus-all (1vsALL) and one-versus-one (1vs1) matrices \cite{hsu2002comparison}.
Because they have no error-correcting ability, later on, channel codes with large hamming distances between the codewords were tried, which is known as problem-independent codings \cite{dietterich1995solving}. However, unlike channel codes in communication, the ``channels'' in ECOC are influenced by the bipartitions of classes: if the classes are partitioned improperly, the ``noise'' (errors) of the channels may be rather high. Furthermore, because there are only $2^{P-1}-1$ possible bipartitions in any binary codes, the code length is limited when $P$ is small \cite{prior2005over}. Finally, the error-correcting ability of ECOC is severely limited. Until now, to our knowledge, few evident proofs showed the error-correcting ability \cite{masulli2000effectiveness}, and in most cases, 1vsALL and 1vs1 are still prevalent \cite{rifkin2004defense}. Although Tapia {\it{et al.}} declared improved performance with low-density parity-check codes and special bipartitions \cite{tapia2004beyond,tapia2010recursive}, we do not know how much the codes contribute to the improvement compared to the bipartitions.

Therefore, ECOC is more properly viewed as a bridge between powerful dichotomizers and multiclass problems without considering the error-correcting ability much, which results in the following three types of problem-dependent codings:

{{The first type learns ECOC in a single objective.}} Because finding an optimal binary coding matrix in a single objective is {\it{NP-complete}}, researchers relaxed the binary coding matrix to a \textit{continuous} one and reformulated the problem to a regularized optimization problem. Typical methods include multiclass-SVM \cite{crammer2002learnability} and several large margin related works \cite{fung2005multicategory,weston1999support,guermeur2002combining,lee2004multicategory,chen2010multiclass,ghorai2010discriminant}
    However, it is worthy noting that multiclass-SVM does not perform better than 1vsALL and 1vs1, and even suffers from longer training time \cite{hsu2002comparison}.
Motivated from multiclass-SVM \cite{crammer2002learnability}, in \cite{zhong2010learning}, Zhong {\it{et al.}} further took base dichotomizers into optimization. Because the objective is too complicated, it has to be solved approximately via the non-convex Constrained Concave-Convex Procedure (CCCP) \cite{smola2005kernel,yuille2003concave}. Moreover, the continuous coding matrix has to be normalized after each CCCP iteration, making the convergence of the objective unguaranteed.
    Summarizing the aforementioned, it might be difficult and time consuming to learn a problem-dependent coding matrix in a single objective.

{{The second type uses ternary codes.}} (i) In \cite{allwein2001reducing}, Allwein {\it{et al.}} extended binary coding to ternary coding, i.e. $\M\in\{-1,0,1\}^{P\times Q}$, see Fig. \ref{fig:fig1} for an example. The entry $M(p,q)=0$ indicates that the $q$-th dichotomizer does not take the $p$-th class into training. This method greatly enlarges the number of all possible bipartitions and makes each binary-class problem easily solved. (ii) In \cite{pujol2006discriminant}, Pujol {\it{et al.}} proposed DECOC which embeds a binary decision tree into the ternary code and takes the bipartition that maximizes the mutual information as a new node of the tree whenever adding a new node to the tree. In \cite{yang2011hierarchical}, Yang and Tsang further proposed to find the most discriminative bipartition in terms of maximum
separating margin. These methods need at most $P-1$ dichotomizers. (iii) To overcome the weakness of decision tree that the nodes of a tree cannot rectify misclassified examples made by their father nodes,
 in \cite{escalera2006ecoc,pujol2008incremental}, Escalera {\it{et al.}} and Pujol {\it{et al.}} proposed ECOC-ONE which iteratively adds dichotomizers that discriminate the most confusing pairs.
(iv) To overcome the weakness of ECOC-ONE that the training process may be blocked by some stubborn binary problems, in \cite{escalera2008subclass}, Escalera {\it{et al.}} further proposed subclass-ECOC, which splits the most confusing class to several subsets (called subclasses) by a decision tree. Because it is also hard to decide when to stop splitting, in \cite{escalera2008subclass}, Escalera {\it{et al.}} used three hyperparameters to control the splitting process, and in \cite{bouzas2011optimizing}, Bouzas {\it{et al.}} tried to find the optimal hyperparameters by searching the hyperparameter spaces.

   \begin{figure}[!t]
\centerline{
\includegraphics[width=85mm]{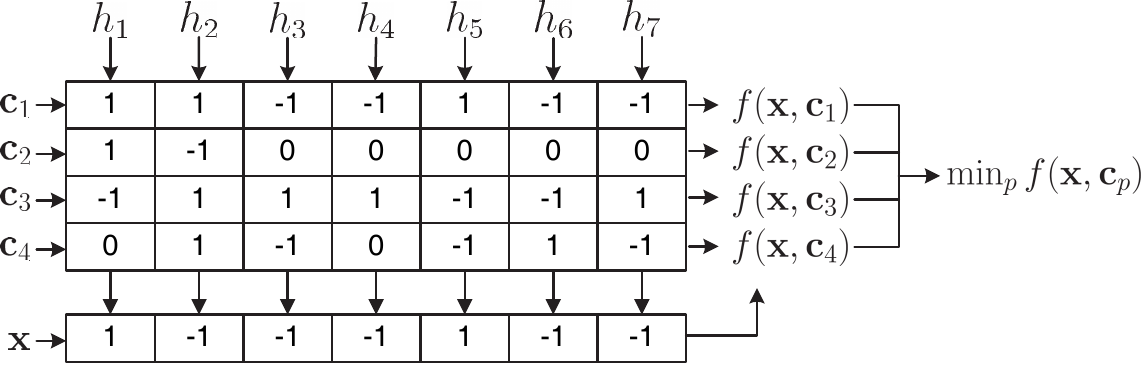}}
\caption{{Coding matrix $\M$ of a ternary ECOC \cite{escalera2010decoding}. In the coding phase, if the entry of $\M$, denoted as $m_{p,q}$, equals to $1$, the dichotomizer $h_q$ takes class $p$ as part of the positive superclass. If $m_{p,q} = -1$, $h_q$ takes class $p$ as part of the negative superclass. If $m_{p,q} = 0$, $h_q$ does not take class $p$ into training \cite{allwein2001reducing}. In the decoding phase, taking a test example $\bm\rho$ into $h_1,\ldots,h_Q$ successively gets a test codeword of $\bm\rho$, denoted as $\x=[x_1,\ldots,x_Q]^T$. Given a decoding strategy $f(\x,\c_p)$, the prediction of $\bm\rho$ can be formulated as a minimization problem $\min_{\c_p\in\mathcal{M}}f(\x,\c_p)$, where $\mathcal{M}=\{\c_p\}_{p=1}^{Q}$ is the set of codewords.}}
\label{fig:fig1}
\end{figure}

{{The third type focuses on improving the diversity between base dichotomizers.}} (i) The following methods improve the diversity by manipulating output codes. In \cite{kuncheva2003measures,kuncheva2005using,escalera2009separability}, Kuncheva {\it{et al.}} and Escalera {\it{et al.}} designed new decoding metrics between codewords. In \cite{escalera2009recoding}, Escalera {\it{et al.}} suggested to selectively replace some 0 positions of an original ternary ECOC codes with 1 or $-1$ according to the accuracies of the base learners at the corresponding classes, which enlarges the distance between the codewords. In \cite{escalera2007boosted}, Escalera {\it{et al.}} combined multiple different DECOC trees. In \cite{hatami2011thinned}, Hatami tried to delete the columns of a coding matrix that have weak diversities.
(ii) Other types of diversity were seldom explored: only in \cite{prior2005over}, Prior and Windeatt manipulated different parameter settings of multilayer perceptrons; in \cite{bagheri2012subspace,bagheri2012rough}, Bagheri \textit{et al.} trained different base dichotomizers with different feature subsets. Our LC-ECOC---a method of manipulating training examples---was partially motivated from this fact.

There are also many other ECOC coding designs and applications, such as the evolution computing based methods \cite{bautista2010compact,bautista2010compact2}, probability ECOC \cite{zhou2011research}, structured outputs of ECOC \cite{kajdanowicz2011multiple}, online ECOC \cite{escalera2010adding,escalera2010online}, and reject rule based ECOC which rejects to use extremely confusing examples \cite{marrocco2007embedding,simeone2012design}.

\subsection{Survey on the Decoding Phase}
The representative decoding methods are HD, ED, probabilistic \cite{passerini2004new}, LB \cite{allwein2001reducing}, and LW decodings \cite{escalera2010decoding}. Here, we focus on reviewing LW decoding since it has a compact theory and performs better than other decoding methods in practice.

 In \cite{escalera2010decoding} and its previous works \cite{escalera2006ecoc,escalera2009recoding}, Escalera {\it{et al.}} argued that a good decoding strategy should make all codewords have the same decoding {\it{dynamic range}} and zero decoding {\it{dynamic range bias}}. Based on the argument, they proposed the LW decoding for ternary ECOCs, which is the first decoding strategy of ternary ECOCs that satisfies the aforementioned two goals. The LW decoding introduces a predefined weight matrix $\W=[\w_1^T,\ldots,\w_P^T]^T=\left[\renewcommand{\arraystretch}{0}\renewcommand{\arraycolsep}{0pt}\begin{array}{ccc}w_{1,1}&\ldots&w_{1,Q}\\ \vdots&\ddots&\vdots\\w_{P,1}&\ldots&w_{P,Q}\end{array}\right]\in\mathcal{W}$ that has the same size as $\M$ and satisfies the following two constraints:
 \begin{eqnarray}
w_{p,q}&&\left\{\begin{array}{ll}=0&,\mbox{ if } m_{p,q}=0\\\in[0,1]&,\mbox{ if }m_{p,q}\neq0\end{array}\right., \nonumber\\
&&\qquad\quad\forall p=1,\ldots,P, \forall q=1,\ldots,Q\\
&&\sum_{q=1}^{Q}w_{p,q} = 1,\quad \forall p =1,\ldots,P
 \label{eq:2}
 \end{eqnarray}
where $m_{p,q}$ is an element of $\M$ and $\mathcal{W}$ is the set of all feasible weight matrices (i.e., $\W\in\mathcal{W}$).
 When $m_{p,q}\neq0$, $w_{p,q}$ is assigned empirically according to the training accuracy of the $q$-th base dichotomizer on the $p$-th class.

 The prediction function of the LW decoding is given by
 \begin{eqnarray}
\min_{\c_p\in\mathcal{M}}f_{LW}(\x,\c_p)=\min_{\c_p\in\mathcal{M}}\sum_{q=1}^{Q}w_{p,q}\ell(x_{q}c_{p,q})
 \label{eq:3}
 \end{eqnarray}
 where $\ell(\cdot)$ is a user defined loss function, such as the linear loss function $\ell(\theta)=-\theta$.

\section{LC-ECOC}\label{sec:LC-ECOC}
In this section, we first review the layered clustering-based approach for classifier ensembles \cite{rahman2011novel}, and then propose a new LC-ECOC.

\subsection{Layered Clustering-Based Approach}
The layered clustering-based approach \cite{rahman2011novel} is an ensemble learning method that manipulates training examples for enlarging diversity. Specifically, it first splits training examples to several non-overlapping regions by clustering, where the classification problem in each region is further solved by a classifier. The classifiers in all regions group to a super-classifier. Then, it repeats the above procedure several times. Each independent repeat forms a layer of super-classifier. All layers of super-classifiers vote for a test example.

 This method contains two complementary properties. First, the clustering-based approach can identify overlapping patterns that are hard to differentiate, so that the classifier in each layer may achieve a high accuracy. But the clustering-based approach do not include any mechanism to incorporate diversity. Second, the layered approach uses the mechanism of bagging to achieve diversities between the super-classifiers. This layered structure, as proved in \cite[page 2]{dietterich2000ensemble} (an article appeared before \cite{rahman2011novel}), will improve the discriminability of a classifier ensemble on a given binary-class problem.

\subsection{LC-ECOC}
Motivated by ECOC-ONE \cite{pujol2008incremental} and subclass-ECOC \cite{escalera2008subclass}, the proposed LC-ECOC also uses the greedy training strategy, a strategy that iteratively adds new dichotomizers that intend to solve the most difficult binary-class problems of previous iterations. The difference between them lies on how they deal with the ``\textit{stubborn}'' binary-class problems, where ``stubborn'' means that a binary-class problem has been tried to solve by a dichotomizer, but it appears to be the most difficult problem again.
 When such a situation happens, ECOC-ONE has to stop training, subclass-ECOC employs a decision-tree to further split the problem, and our LC-ECOC trains one layer of clustering-based dichotomizer \cite{rahman2011novel} on the problem.
 Because different layers of clustering-based dichotomizers are different in terms of errors, LC-ECOC will not be blocked by the stubborn problems.

Figure \ref{fig:fig4x} gives an example of LC-ECOC for a three-class classification problem. It is initialized with a compact code $\M$. At the first iteration, it finds the most difficult binary-class problem, supposing to be $\mathbf{m}=[1,-1,0]^T$. Because $\mathbf{m}$ is not a column of $\M$, LC-ECOC trains a simple base dichotomizer $h_{3}^{(s)}$ to discriminate classes 1 and 2. At the second iteration, when observing the fact that the most difficult problem $[1,-1,0]^T$ has already appeared as the third column of $\M$. it trains one layer of clustering-based dichotomizer $h_{4}^{(c)}$, so as to $h_{5}^{(c)}$.

  \begin{figure}[!t]
\centerline{
\includegraphics[width=38mm]{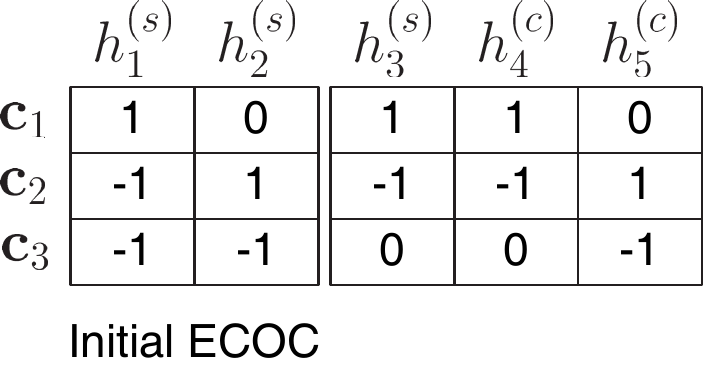}}
\caption{{An example of LC-ECOC for a three-class classification problem. $h^{(s)}$ indicates a simple dichotomizer. $h^{(c)}$ indicates a clustering-based dichotomizer.}}
\label{fig:fig4x}
\end{figure}

We adopt the heterogeneous clustering-based approach \cite{rahman2011novel,verma2011cluster} to train each complicated clustering-based dichotomizer (Algorithm \ref{chaECOC:alg:1}). Specifically, in the training process, the heterogeneous clustering-based approach splits the space of a pair of classes to $N_c$ regions ($N_c>1$) without considering the class attributes. For each region, if the region contains examples from both classes, it trains a simple base dichotomizer on the region; otherwise, it remembers the class attribute of the region. In the prediction process, a test example is first assigned to its \textit{host region}, a region whose center has the minimum Euclidean distance from the example over all regions. Then, if the region owns a base dichotomizer, the approach predicts the test example by the base dichotomizer; otherwise, it assigns the class attribute of the region to the test example.

\begin{algorithm}[t]
    \caption{LC-ECOC.}
    \begin{algorithmic}[1]
    \label{chaECOC:alg:1}
    \STATE /* Training */
\REPEAT
\STATE Find the most confusing pair of classes\\
\IF{the pair has not been tried to solve by ECOC}
\STATE Train a simple dichotomizer for the pair\\
\ELSE
\STATE /* Training a clustering-based dichotomizer */
\STATE Partition the space of the pair to $N_c$ regions by clustering
\FOR{$i = 1,\ldots,N_c$}
    \IF{the examples in the $i$-th region are from both classes}
        \STATE Train a dichotomizer on the region\\
    \ELSE
        \STATE Remember the class attribute of the region
    \ENDIF
\ENDFOR
\ENDIF\\
\STATE Add the new dichotomizers to the ECOC ensemble\\
\UNTIL{the training risk converges}
\STATE $ $
 \STATE /* Prediction */
 \FOR{$q=1,\ldots,Q$}
     \IF{the dichotomizer $h_q$ is a simple one}
     \STATE Predict the example by $h_q$
     \ELSE
         \STATE Assign the test example to its host region
         \IF{the region owns a dichotomizer}
         \STATE Predict the example by the dichotomizer of the region
         \ELSE
         \STATE Assign the class attribute of the region to the example
         \ENDIF
    \ENDIF
\ENDFOR
\STATE Decode the predicted codeword of the example
\end{algorithmic}
\end{algorithm}

  \begin{figure}[!t]
\centerline{
\includegraphics[width=90mm]{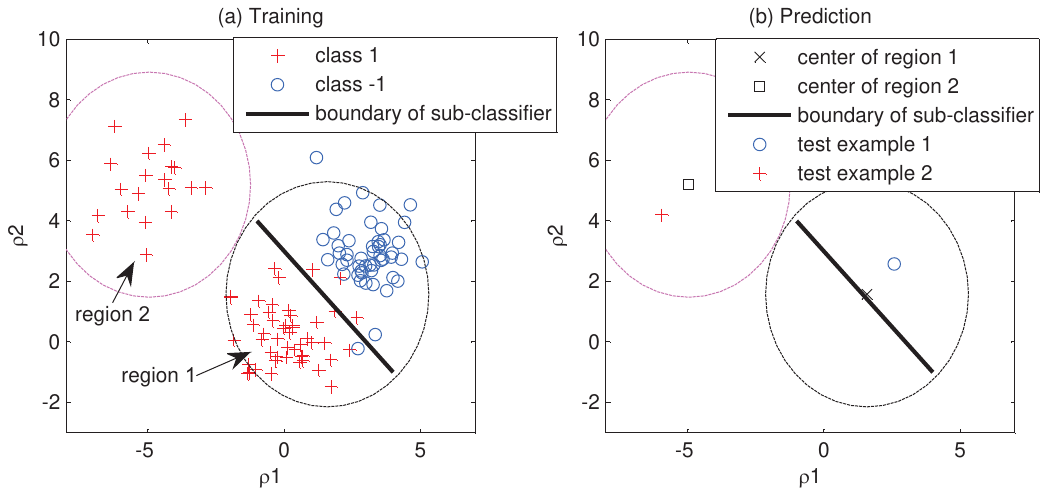}}
\caption{{An example of the heterogeneous clustering-based dichotomizer.}}
\label{fig:fig_xxx}
\end{figure}

Figure \ref{fig:fig_xxx} gives an example of the training and prediction of a heterogeneous clustering-based dichotomizer. In the training process (Fig. \ref{fig:fig_xxx}a), it first finds the most confusing region by spliting the training examples to two regions by $k$-means. Because region 1 consists of two classes, it trains a simple dichotomizer to discriminate the two classes in the region. Because region 2 consists of only class 1, it simply remembers the class attribute. In the prediction process (Fig. \ref{fig:fig_xxx}b), because example 1 falls into region 1, it classifies example 1 to class $-$1 by the simple dichotomizer in region 1. Because example 2 falls into region 2 and because region 2 belongs to class 1, it classifies example 2 to class 1.

Note that the clustering algorithms that have high accuracies, such as spectral clustering \cite{shi2002normalized}, agglomerative clustering \cite{yuan2012agglomerative}, and maximum margin clustering \cite{zhang2012linearithmic}, are not suitable for this job. The more ``weak'' and unstable the clustering algorithm is, the more suitable it seems to be. Hence, the traditional $k$-means clustering \cite{mcqueen1967kmeans} is adopted.

\section{CPA Based OW Decoding for ECOC}\label{sec:calcweights}
In this section, we first propose the OW decoding, and then employ CPA to accelerate the decoding algorithm.

\subsection{OW Decoding}\label{subsec:OWdecoding}

OW decoding optimizes the weight matrix of the LW decoding \cite{escalera2010decoding} for the minimal training risk, which is formulated as a {\it{linear programming}} problem that can be solved in time $\O(n\log n)$.

The weight matrix is optimized as follows.
Given a training example $\bm\rho_i$ with its predicted codeword from the dichotomizers, denoted as $\x_i$, and ground truth label $y_i$, if $\bm\rho_i$ is classified correctly, according to (\ref{eq:3}), the following criterion is satisfied:
 \begin{eqnarray}
\sum_{q=1}^{Q}w_{y_i,q}\ell(x_{i,q}c_{y_i,q})\le\sum_{q=1}^{Q}w_{p,q}\ell(x_{i,q}c_{p,q}),\nonumber\\
\forall p=1,\ldots,P.
 \label{eq:4}
 \end{eqnarray}
 where $\ell(\theta)$ can be defined as $\ell(\theta)=-\theta$.
Letting $\u_{i,p}=[\ell(x_{i,1}c_{p,1}),\ldots,\ell(x_{i,Q}c_{p,Q})]^T$ can rewrite equation (\ref{eq:4}) as
 \begin{eqnarray}
\w_{y_i}^T\u_{i,y_i}-\w_{p}^T\u_{i,p}\le 0,\quad \forall p=1,\ldots,P.
 \label{eq:5}
 \end{eqnarray}
where any $\u_{i,p}$ should be normalized to $\u_{i,p}/\max_{i,p,q}|u_{i,p,q} |$, so as to prevent unexpected numerical problems.
  If $\bm\rho_i$ is misclassified, it will cause a training loss $\xi_i$. One possible measurement of $\xi_i$ is the {\it{hinge loss}}:
 \begin{eqnarray}
\xi_i=\max_{p=1,\ldots,P}\left(0,\w_{y_i}^T\u_{i,y_i}-\w_{p}^T\u_{i,p}\right).
 \label{eq:6}
 \end{eqnarray}
 Minimizing the training risk is to minimize the sum of the training loss of all examples, which is formulated as the following convex {{linear programming}} problem:
 \begin{eqnarray}
&&\min_{\W\in\mathcal{W}}  \J(\W)\nonumber\\
\triangleq&&\min_{\W\in\mathcal{W}} \sum_{i=1}^{n}\max_{p=1,\ldots,P}\left(0,\w_{y_i}^T\u_{i,y_i}-\w_{p}^T\u_{i,p}\right)
 \label{eq:7}
 \end{eqnarray}
 which can be rewritten as the following constrained optimization problem:
 \begin{eqnarray}
  \label{eq:8}
\min_{\W\in\mathcal{W},\xi_{i}\ge0}&&\sum_{i=1}^{n}\xi_{i}\\
\mbox{subject to } &\mbox{ }&\w_{p}^T\u_{i,p}-\w_{y_i}^T\u_{i,y_i}\ge -\xi_{i},\nonumber\\
&& \forall i=1,\ldots,n,\quad\forall p=1,\ldots,P.\nonumber
 \end{eqnarray}
%Problem (\ref{eq:8}) can be solved globally in time $\O(n\log n)$. Because this method is inspired by the soft-margin SVM \cite{cortes1995support,vapnik2000nature,scholkopf2002learning}, we also call the parameter $\xi_i$ the slack variable.

Note that the definition of $\xi_i$ in (\ref{eq:6}) is important to the difficulty of the optimization. If it is defined as the training error, i.e. $\xi_i\in\{0,1 \}$, problem (\ref{eq:8}) will be an integer matrix optimization problem with an {\it{NP-complete}} complexity. Usually, we use some convex surrogate function, such as hinge loss, to relax $\xi_i$ to a continuous value. As will be shown in Section \ref{sec:mmc_dual}, this relaxation enforces us to pick the most confusing pair of classes according to the \textit{training risk matrix} but not the \textit{confusion matrix} of classification errors.

\subsection{CPA Based OW Decoding}
Because problem (\ref{eq:8}) has $\O(n)$ parameters and $\O(n)$ constraints, solving problem (\ref{eq:8}) is still inefficient for large-scale problems.
Here, we employ the well-known CPA \cite{kelley1960cutting,joachims2006training,teo2007scalable,franc2008optimized,joachims2009cutting} to further lower its time complexity to $\O(n)$.

CPA is an efficient optimization tool for those convex optimization problems with large amounts of constraints. Its time and storage complexities are irrelevant to the number of constraints. In CPA terminology, a problem with a full constraint set is called a master problem \cite{franc2008optimized}, while a problem with only a constraint subset from the full set is called a reduced problem, or a cutting-plane subproblem.
Generally, CPA begins with a reduced problem that has only an empty working constraint set, and then iterates the following two steps: (i) solving the reduced problem with the working constraint set; (ii) adding the most violated constraint of the current solution point from the full set to the working constraint set, so as to form a new reduced problem.
If the new voilated constraint violates the solution of the previous reduced problem by no more than $\epsilon$, CPA is stopped, where $\epsilon$ is a user defined solution precision. It has been proven that the number of iterations is upper bounded by $\O(1/\epsilon)$ \cite{teo2007scalable}, which is irrelevant to $n$.

For our problem, we first reformulate problem (\ref{eq:8}) to the following equivalent optimization problem:
 \begin{eqnarray}
\min_{\W\in\mathcal{W},\xi\ge0}&&\xi\nonumber\\
\mbox{subject to }&\mbox{ }&\sum_{i=1}^{n}\sum_{p=1}^{P}g_{i,p}\left(\w_{p}^T\u_{i,p}-\w_{y_i}^T\u_{i,y_i}\right)\ge -\xi,\nonumber\\
&&\qquad\quad\forall \mathbf{G}\in \mathcal{Z}^n
 \label{eq:9}
 \end{eqnarray}
 where $\g_i = [g_{i,1},\ldots,g_{i,P}]^T$, $ \mathbf{G}=\left[\g_1,\ldots,\g_n\right]=\left[\renewcommand{\arraystretch}{0}\renewcommand{\arraycolsep}{0pt}\begin{array}{ccc}g_{1,1}&\ldots&g_{n,1}\\ \vdots&\ddots&\vdots\\g_{1,P}&\ldots&g_{n,P}\end{array}\right]$, and the set $\mathcal{Z} = \{\mathbf{z}_p \}_{p=1}^{P}$ with $\mathbf{z}_p$ defined as
 \begin{eqnarray}
z_{p,k}=&&\left\{\begin{array}{ll}1&,\mbox{ if } k=p\\0&,\mbox{ otherwise }\end{array}\right. ,\quad k = 1,\ldots,P.
 \label{eq:11}
 \end{eqnarray}
Problem (\ref{eq:8}) and problem (\ref{eq:9}) are equivalent in the following theorem.
 \begin{thm}\label{thm:1}
   Any solution $\W$ of problem (\ref{eq:9}) is also a solution of problem (\ref{eq:8}), and {\it{vice versa}}, with $\xi=\frac{1}{n}\sum_{i=1}^{n}\xi_i$.
 \end{thm}
 \begin{proof}
See Appendix \ref{app:proof_thm1}.
\end{proof}

Comparing problem (\ref{eq:9}) to (\ref{eq:8}), we can see that although problem (\ref{eq:9}) has only 1 slack variable, the number of its constraints is as high as $P^n$. Fortunately, problem (\ref{eq:9}) can be solved approximately by CPA. The CPA based OW decoding algorithm is described in Algorithm \ref{alg:2}. The derivation, which is omitted here, is similar to the well-known SVM$^{\scriptsize\mbox{perf}}$ toolbox \cite{joachims2006training,joachims2009cutting,joachims2009sparse}.

\begin{algorithm}[t]
    \caption{CPA based OW decoding.}
    \begin{algorithmic}[1]
    \label{alg:2}
\REQUIRE  Dataset $\mathcal{U} = \left\{\{\mathbf{u}_{i,p}\}_{p=1}^{P},y_i \right\}_{i=1}^{n}$.\\
\ENSURE Optimal weight matrix $\W$.
\renewcommand{\algorithmicrequire}{\textbf{Initialization: }}
\REQUIRE{Arbitrary initial weight matrix $\W$ ($\W\in\mathcal{W}$), empty initial working constraint set $\bm{\Omega}=\{\}$, the size of working constraint set $|\Omega|\gets0$.
 }
\REPEAT
\STATE $|\Omega| \gets|\Omega|+1$\\
\STATE Calculate the most violated constraint $\G_{|\Omega|}$
\begin{eqnarray*} g_{i,p}^{|\Omega|} \gets \bigg\{\begin{array}{cl} 1, &\mbox{ }\mbox{if } p = \arg\max_{p}\left( \w_{y_i}^T\u_{i,y_i} - \w_p^T\u_p \right)\\ 0,& \mbox{ otherwise} \end{array}\nonumber\end{eqnarray*}\\
\STATE Add the most violated constraint $\G_{|\Omega|}$ to $\bm{\Omega}$:
\begin{eqnarray*} \bm{\Omega}\gets\bm{\Omega}\cup\G_{|\Omega|}\nonumber\end{eqnarray*}\\
\STATE Solve the reduced problem
\begin{eqnarray}\label{eq:alg} \min_{\W\in\mathcal{W},\xi\ge0}&& \xi\\
\mbox{subject to} &\mbox{ }& \sum_{i=1}^{n}\sum_{p=1}^{P}g_{i,p}\left(\w_{p}^T\u_{i,p}-\w_{y_i}^T\u_{i,y_i}\right)\ge-\xi,
\nonumber\\
&&\quad\quad\qquad\forall\mathbf{G}\in\bm{\Omega}\nonumber \end{eqnarray}
\UNTIL{$\bm{\Omega}$ is unchanged}
\end{algorithmic}
\end{algorithm}

Because problem (\ref{eq:alg}) has very few constraints, the time complexity of Algorithm 2 is $\O(n)$, which is consumed on calculating $\sum_{i=1}^{n}g_{i,p}\u_{i,p}$ in (\ref{eq:alg}).
Besides the linear time complexity, the CPA based OW decoding has another important merit: its storage complexity is irrelevant to the implementation method of the linear programming toolbox, since the linear programming problem (\ref{eq:alg}) has only $\O(1)$ parameters and $\O(1)$ constraints. We take the standard linear programming toolbox in MATLAB as an example: if we rewrite both Eqs. (\ref{eq:8}) and (\ref{eq:alg}) to the standard form ``$\min_{\x} \mathbf{f^Tx}\mbox{ subject to }\mathbf{Ax}\le \mathbf{b}$'', matrix $\mathbf{A}$ in (\ref{eq:8}) is $(PQ+n)\times n$ in size, while $\mathbf{A}$ in (\ref{eq:alg}) is only $(PQ+1)\times |\Omega|$ in size where $ |\Omega|$ denotes the size of the working constraint set and is a small integer that is irrelevant to $n$. As a result, the original OW decoding cannot handle middle scale datasets in the MATLAB environment, while the CPA based OW decoding is not limited by the scale of the dataset.

\section{WOLC-ECOC}\label{sec:mmc_dual}

\begin{algorithm}[t]
    \caption{WOLC-ECOC.}
    \begin{algorithmic}[1]
    \label{alg:WOLC-ECOC}
\REQUIRE  Dataset $\mathcal{D} = \{\bm\rho_i,y_i \}_{i=1}^{n}$, the number of the most confusing pairs per iteration $s$, maximum iteration number $T$, solution precision $\eta$, parameter for the termination condition $Z$.\\
\ENSURE {ECOC coding matrix $\M_{o}$ and the corresponding classifier ensemble $\mathcal{C}_{o}$}, optimal weight matrix $\W_{o}$.
\renewcommand{\algorithmicrequire}{\textbf{Initialization: }}
\REQUIRE{initial ternary ECOC coding matrix $\M\in\{-1,0,1\}^{P\times Q}$ and the classifier ensemble $\mathcal{C}=\{h_1,\ldots,h_Q\}$ that is learned from $\M$ and $\mathcal{D}$, $\J_o'\gets inf$, $z\gets0$, $t\gets0$.
 }
\REPEAT
\FOR{$i=1,\ldots,n$}
    \STATE Predict $\bm\rho_i$ by the LC-ECOC prediction process
    \STATE Calculate $\{\u_{i,p}\}_{p=1}^{P}$ defined in Eq. (\ref{eq:4})
\ENDFOR
\STATE{/* Optimize weight matrix */}\nonumber\\
\STATE $\{\W,\J_o\}\gets${\it{WeightOptimization}}$(\mathcal{U}, \M)$, where $\mathcal{U} = \left\{\{\mathbf{u}_{i,p}\}_{p=1}^{P},y_i \right\}_{i=1}^{n}$ \\
\IF{$\J_o =0$}
\STATE $\M_{o} \gets \M$, $\mathcal{C}_{o}\gets \mathcal{C}$, $\W_{o}\gets\W$\\
\RETURN
\ENDIF
\STATE{/* Get the most confusing pairs */}\\
\STATE Find $s$ pairs of classes that have the highest training risks $\{\m_k\}_{k=1}^{s}$. Get their corresponding training risks $\{\epsilon_k\}_{k=1}^{s}$ \\
\STATE{/* Learn the base dichotomizers from $\{\m_k\}_{k=1}^{s}$ */}\\
\FOR{$k=1,\ldots,s$}
\IF{$\epsilon_k\neq 0$}
\IF{$\m_k$ does not equal to any column of $\M$}
\STATE $h_k'\gets$  {\it{SimpleLearning}}($\mathcal{D},\m_k$)\\
\ELSE
\STATE $h_k'\gets$ {\it{ClusteringBasedLearning}}($\mathcal{D},\m_k$)\\
\ENDIF\\
%{\color{green}/* Update the ECOC ensemble */}  \\
\STATE $\M\gets [\M, \m_k]$, $\mathcal{C}\gets\mathcal{C}\cup h_k' $\\
\ENDIF
\ENDFOR\\
\STATE{/* Control the termination criterion */}  \\
\IF{$(\J_o'-\J_o)/\J_o \le \eta $}
\STATE $z \gets z+1$\\
\ELSE
\STATE $z\gets0$\\
\STATE $\M_{o} \gets \M$, $\mathcal{C}_{o}\gets \mathcal{C}$, $\W_{o}\gets\W$\\
\ENDIF
\STATE $t\gets t+1$, $\J_o'\gets\J_o$\\
\UNTIL{ $z=Z$ or $t=T$}
\end{algorithmic}
\end{algorithm}

The framework of WOLC-ECOC is presented in Fig. \ref{fig:fig_visio}.
The training procedure of WOLC-ECOC is detailed in Algorithm \ref{alg:WOLC-ECOC} and described as follows.

 WOLC-ECOC starts with any valid ECOC $\{\M,\mathcal{C}\}$ with $\mathcal{C}=\{h_1,\ldots,h_Q\}$, such as 1vsALL, 1vs1, or compact code (i.e., $Q<P$), and then iterates the following two steps:
\begin{enumerate}[(i)~]
\item The first step optimizes the weight matrix $\W$ of the OW decoding and obtains the minimal training risk $\J_o$ by the {\it{WeightOptimization}} function which is described in Section \ref{sec:calcweights}.

\item The second step first finds the top $s$ most confusing pairs of classes, denoted as $\{\m_k\}_{k=1}^{s}$, and then adds all $s$ dichotomizers $\{h_k'\}_{k=1}^{s}$ that discriminate $\{\m_k\}_{k=1}^{s}$ respectively to $\mathcal{C}$. For training $h_k'$, as presented in LC-ECOC (Algorithm \ref{chaECOC:alg:1}), two situations should be considered: if $\m_k$ does not equal to any column of $\M$, we train a new simple dichotomizer ${h_k'}^{(s)}$ as usual by the {\it{SimpleLearning}} function; otherwise, we train a complicated clustering-based dichotomizer ${h_k'}^{(c)}$ by the {\it{ClusteringBasedLearning}} function in Section \ref{sec:LC-ECOC}.
\end{enumerate}

The loop stops when the maximum iteration number $T$ is reached or the following inequality is satisfied for continuous $Z$ iterations:
 \begin{eqnarray}
 \frac{\J_o'-\J_o}{\J_o} \le \eta
 \label{eq:termination}
 \end{eqnarray}
 where $\J_o$ and $\J_o'$ are the training risks of the current and previous iterations respectively, and $\eta$ is a user defined solution precision.
 Finally, the ECOC ensemble $\{\M_{o}, \mathcal{C}_{o},\W_{o}\}$ that achieves the minimum risk is returned.
  Here, we have to note that although OW decoding can reach its global minimum solution at each WOLC-ECOC iteration, the overall heuristic training process only reaches a local minimum solution.

 WOLC-ECOC has two merits when compared to its components.
First, the monotonic decrease of the training risk of WOLC-ECOC is guaranteed, see Appendix \ref{app:proof_thm2} for the proof. Second, a small ECOC code length is ensured, since discriminating the most difficult binary-class problem at each iteration make ECOC obtain the maximum performance gain.

 In Algorithm \ref{alg:WOLC-ECOC}, we have considered the following three issues for the robustness and efficiency of WOLC-ECOC.

{First, how to balance the discriminability and the code length?} Multiple layers of clustering-based dichotomizers might trigger a significant performance improvement with a risk of overfitting, while one or two layers might not improve the performance.
   To solve the problem, the following termination criterion is used: if the training risk does not decrease in a rate of $\eta$ ( in (\ref{eq:termination})) for $Z$ continuous iterations, we stop the training procedure. Usually, setting $Z$ to an arrange of 3 to 5 is enough.

   {Second, how to make the performance robust? }Sometimes, the most confusing pair is too stubborn to overcome. To prevent this unwanted situation, we discriminate the top $s$ most confusing pairs of classes, denoted as $\{\m_k\}_{k=1}^{s}$, instead of a single most confusing pair.

{Third, how to define the most confusing pair of classes?}
ECOC-ONE \cite{pujol2008incremental} selects the most confusing pair of classes by the confusion matrix $\bm\epsilon$ which is defined as
 \begin{eqnarray}
\epsilon_{i,j}& =& \sum_{k:\pmb\rho_k\in \footnotesize{\mbox{ class }} i}  e_{i,j}(\bm\rho_k)
 \end{eqnarray}
 where function $e_{i,j}(\cdot)$ is defined as
 \begin{eqnarray}
e_{i,j}(\bm\rho)= \left\{ \begin{array}{l}
1, \quad\mbox{  if } \bm\rho\in \mbox{class } i \mbox{ but is misclassified to } j ,\\
0, \quad\mbox{  otherwise.}
\end{array}
\right.\nonumber
 \end{eqnarray}

   \begin{figure}[!t]
\centerline{
\includegraphics[width=70mm]{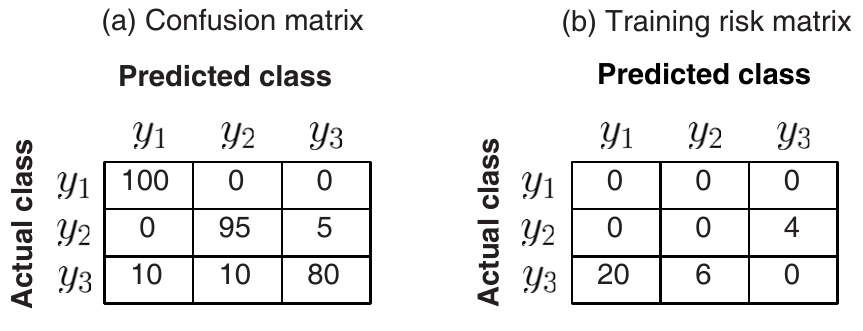}}
\caption{{Comparison of the confusion matrix and the training risk matrix of a three-class classification problem.}}
\label{fig:fig3}
\end{figure}
However, because OW decoding relaxes the loss function from classification error $\{0,1 \}$ to a convex continuous surrogate function (\ref{eq:6}) with a range of $[0,+\infty)$, Algorithm \ref{alg:WOLC-ECOC} minimizes the training risk $J(\W)$ instead of classification error. That is to say, for each iteration, Algorithm \ref{alg:WOLC-ECOC} picks a pair of classes that has the highest training risk but not the one that has the highest classification error. Correspondingly, the training risk matrix $\bm\epsilon$ is defined as
 \begin{eqnarray}
\epsilon_{i,j}& =& \sum_{k:\pmb\rho_k\in \footnotesize{\mbox{ class }}i}  \left(\w_{i}^T\u_{k,i}- \w_{j}^T\u_{k,j}\right)\cdot\nonumber\\
&&\qquad\qquad \delta\left( \min_{p=1,\ldots,P;p\neq j}\w_{p}^T\u_{k,p}-  \w_{j}^T\u_{k,j}\right)
 \label{eq:misclassificationerror}
 \end{eqnarray}
where $\delta(\cdot)$ is the indicator function:
 \begin{eqnarray}
\delta(a) = \left\{ \begin{array}{l}
1, \quad\mbox{  if } a>0,\\
0, \quad\mbox{  otherwise.}
\end{array}
\right.\nonumber
 \end{eqnarray}
  An example comparison between the confusion matrix and the training risk matrix is shown in Fig. \ref{fig:fig3}. From Fig. \ref{fig:fig3}a, we observe that (i) each class consists of 100 examples; (ii) the candidate confusing pairs of classes are $\m^{1,2}=[1,-1,0]^T$, $\m^{1,3}=[1,0,-1]^T$, and $\m^{2,3}=[0,1,-1]^T$ with the numbers of misclassified examples being $\epsilon_{1,2}=0+0=0$, $\epsilon_{1,3}=10+0=10$, and $\epsilon_{2,3}=10+5=15$ respectively; (iii) the most confusing pair is selected as $\m^{2,3}$.

  But from Fig. \ref{fig:fig3}b, we observe that (i) the training risk pairs are $\epsilon_{1,2}=0+0=0$, $\epsilon_{1,3}=0+20 = 20$, and $\epsilon_{2,3}=4+6 = 10$ respectively; (ii)  the highest training risk pair is $\mathbf{m}^{1,3}$. Comparing Fig. \ref{fig:fig3}a with Fig. \ref{fig:fig3}b, we can see that different optimization objectives might give a binary-class problem different training priorities.

\section{Experimental Analysis}\label{sec:experiments}
In this section, we first compare WOLC-ECOC with 15 coding-decoding pairs on 14 UCI benchmark datasets with 2 kinds of base dichotomizer---AdaBoost and SVM, then study the convergence behavior of WOLC-ECOC, and finally apply WOLC-ECOC to a music genre classification problem.

\subsection{Experimental Settings}\label{subsec:experimental_settings}
We used 14 multiclass datasets in the UCI Machine Learning Repository database\footnote{\url{http://archive.ics.uci.edu/ml/}}.
The properties of the datasets are listed in Table \ref{tab:contentCompare1}.
All datasets were normalized into the range of [0,1] in dimension \cite{hsu2003practical}.

 \begin{table} [t]
\caption{\label{tab:contentCompare1} {Descriptions of the datasets. ``$n$'' is the dataset size, ``$d$'' is
the dimension, ``$P$'' is the number of the classes.}}
%\vspace{2mm}
\centerline{
\scalebox{1}{
%\begin{tabular*}{170mm}{|@{}l@{\extracolsep{\fill}}|c|c|c|c|c|c|c|c|c|c|c|c|c|c@{\extracolsep{\fill}}|} %{l*{14}{|c}}
\begin{tabular}{l|l||c|c|c}
 \hline
ID & Data & \textbf{$n$} & \textbf{$d$} & \textbf{$P$} \\
 \hline
 \hline
1&{Dermathology} & 366&34& 6\\
\hline
2&{Iris} & 150& 4&3\\
 \hline
3&{Ecoli}&336&7& 8 \\
 \hline
4& {Wine} &178&13&3\\
 \hline
5& {Glass} &214&9&7\\
\hline
6& {Thyroid} & 215&5&3\\
\hline
7& {Vowel} & 990&10&11 \\
\hline
8& {Balance} & 625& 4&3 \\
\hline
9& {Yeast} & 1484&8& 10 \\
\hline
10& {Satimage} & 6435& 36& 7\\
 \hline
11& {Pendigits} & 10992&16& 10 \\
 \hline
12&{Segmentation} & 2310&19& 7 \\
 \hline
13& {OptDigts} & 5620&64& 10 \\
\hline
14& {Vehicle} & 846& 18& 4\\
\hline
\end{tabular}}}
\end{table}

For the proposed WOLC-ECOC, the number of the most confusing pairs per iteration $s$ was set to 3. The termination condition $Z$ was set to 3. The solution precision $\eta$ was set to 0.01. The initial ECOC was 1vsALL. The maximum iteration number $T$ was set to $3P$ where $P$ is the number of classes.

To show the effectiveness of WOLC-ECOC, we compared it with 5 state-of-the-art ECOC coding designs, including 1vs1, 1vsALL, random ECOC\cite{allwein2001reducing}, DECOC \cite{pujol2006discriminant}, and ECOC-ONE using 1vsALL as its initialization \cite{escalera2006ecoc}. Each of the comparison coding methods combined 3 decoding methods, including HD, LB \cite{allwein2001reducing}, and LW \cite{escalera2010decoding} decodings. We followed the ECOC library \cite{escalera2010error}\footnote{\url{http://sourceforge.net/projects/ecoclib/}} for the implementations of the referenced methods.

To demonstrate how a base classifiers affects the performance, we used two popular base classifiers---discrete AdaBoost \cite{freund1996experiments} and Gaussian RBF kernel based SVM \cite{joachims2009cutting}\footnote{\url{http://svmlight.joachims.org/svm_perf.html}}. AdaBoost uses 40 {\it{decision stump}} weak learners. The parameters of SVM were searched in grid: parameter $C$ was searched through $\{2^{12},2^{13},\ldots,2^{18}\}$, and the kernel width $\sigma$ of the RBF kernel was searched through $\{0.25{\gamma},0.5{\gamma},{\gamma},2{\gamma},4{\gamma} \}$, where $\gamma$ is the average Euclidean distance between the training examples.

For each dataset, we ran each pair of the coding-decoding methods 10 times and recorded the average experimental results. For each single run, we applied a \textit{stratified sampling} and \textit{ten-fold cross-validation}, and tested for confidence interval at 95\% with the two-tailed {\it{t}} test. Therefore, we conducted 100 independent runs on each dataset for each pair of coding-decoding methods. %Because the experiment runs with 16 pairs of coding-decoding methods and 2 kinds of well-tuned dichotomizers on 14 UCI datasets, we need to conduct 44800 independent runs in total, which is a large-scale experiment that is sufficient to evaluate the effectiveness and efficiency of WOLC-ECOC.

\subsection{Effectiveness}

Tables \ref{tab:acc1} and \ref{tab:acc2} list the classification accuracies of all coding-decoding methods with respect to AdaBoost and SVM respectively. From Table \ref{tab:acc1}, we can see clearly that WOLC-ECOC is the most effective one. But from Table \ref{tab:acc2}, we observe that WOLC-ECOC is less effective than the 1vs1 coding but more effective than other coding methods.

 \begin{table*} [t]
\caption{\label{tab:acc1} {Accuracy comparison (\%) of the ECOC coding-decoding methods on the UCI datasets. The base learner is the {discrete AdaBoost}. In each grid, the first line is the accuracy and the second line is the standard deviation.
The row ``\textbf{Rank}'' is the average rank over all 14 datasets.}}
%\vspace{2mm}
\centerline{
\scalebox{0.86}{
%\begin{tabular*}{170mm}{|@{}l@{\extracolsep{\fill}}|c|c|c|c|c|c|c|c|c|c|c|c|c|c@{\extracolsep{\fill}}|} %{l*{14}{|c}}
\begin{tabular}{l||ccc|ccc|ccc|ccc|ccc|c}
 \hline
\textbf{Coding} &\multicolumn{3}{c|}{1vs1}	&\multicolumn{3}{c|}{1vsALL	}&\multicolumn{3}{c|}{Random}	 &\multicolumn{3}{c|}{ECOC-ONE}	&\multicolumn{3}{c|}{DECOC}	&{WOLC-ECOC} \\
\hline
\textbf{Decoding}& HD & LB & LW & HD & LB & LW& HD & LB & LW& HD & LB & LW& HD & LB & LW& OW\\
\hline
\hline
 \multirow{2}*{{Dermathology}}	&91.11&	91.11&	\textbf{92.18}&	87.51&	87.51&	89.44&	81.06&	80.86&	82.47&	89.10&	89.23&	91.86&	70.35&	71.19&	73.16&	91.56  \\
& (0.00)&	(0.00)&	(0.00)&	(0.00)&	(0.00)&	(0.00)&	(2.29)&	(3.91)&	(2.96)&	(0.43)&	(0.49)&	(0.00)&	(2.11)&	(1.87)&	(1.84)&	(0.22)
\\
 \hline
 \multirow{2}*{{Iris}}	&94.64	&94.64	&94.64	&\textbf{96.73	}&\textbf{96.73}	&96.03	&96.03	&95.96	&95.89	&95.34	&95.34	&95.62	&96.03	&96.03	&96.03	&96.03
 \\
 &(0.00)	&(0.00)	&(0.00)	&(0.00)	&(0.00)	&(0.00)	&(0.46)	&(0.61)	&(0.44)	&(0.00)	&(0.00)	&(0.36)	&(0.00)	&(0.00)	&(0.00)	&(0.00)
\\
 \hline
\multirow{2}*{{Ecoli}}	&85.00	&85.00	&84.75	&81.27	&81.27	&79.99	&76.30	&76.63	&77.52	&80.17	&80.16	&78.84	&75.16	&72.36	&78.47	&\textbf{87.40}
\\
 &(0.00)	&(0.00)	&(0.00)	&(0.00)	&(0.00)	&(0.00)	&(2.35)	&(1.33)	&(1.36)	&(1.18)	&(1.00)	&(0.83)	&(4.19)	&(4.26)	&(2.37)	&(0.82)
\\
 \hline
\multirow{2}*{{Wine}}	&\textbf{94.31}	&\textbf{94.31}	&\textbf{94.31}	&91.44	&91.44	&91.44	&93.27	&93.00	&93.20	&92.05	&91.70	&91.87	&\textbf{93.87}	&93.58	&\textbf{93.93}	&93.69
 \\
&(0.00)	&(0.00)	&(0.00)	&(0.00)	&(0.00)	&(0.00)	&(0.88)	&(0.92)	&(0.62)	&(0.55)	&(0.63)	&(0.68)	&(0.56)	&(0.24)	&(0.69)	&(0.00)
 \\
 \hline
\multirow{2}*{{Glass}}	&67.78	&67.78	&67.38	&57.12	&57.12	&\textbf{68.15}	&61.81	&63.14	&63.63	&60.45	&60.85	&65.00	&58.21	&57.25	&63.48	&67.28
 \\
 &(0.00)	&(0.00)	&(0.00)	&(0.00)	&(0.00)	&(0.00)	&(1.49)	&(3.14)	&(2.50)	&(1.93)	&(2.63)	&(2.20)	&(3.98)	&(4.35)	&(2.55)	&(0.66)
\\
 \hline
\multirow{2}*{{Thyroid}}	&93.45	&93.45	&93.45	&93.95	&93.95	&93.95	&\textbf{94.57}	&94.14	&94.16	&93.95	&93.95	&93.95	&93.78	&93.81	&93.93	&\textbf{95.45}
 \\
 &(0.00)	&(0.00)	&(0.00)	&(0.00)	&(0.00)	&(0.00)	&(0.92)	&(0.93)	&(0.87)	&(0.00)	&(0.00)	&(0.00)	&(0.60)	&(1.05)	&(0.72)	&(0.00)
\\
 \hline
\multirow{2}*{{Vowel}}	&58.74	&58.74	&58.74	&39.80	&39.80	&45.97	&39.58	&37.92	&40.99	&42.60	&42.10	&46.50	&43.24	&45.80	&45.28	&\textbf{60.61}
 \\
 &(0.00)	&(0.00)	&(0.00)	&(0.00)	&(0.00)	&(0.00)	&(2.60)	&(1.67)	&(1.95)	&(1.65)	&(1.11)	&(1.49)	&(2.74)	&(1.99)	&(2.44)	&(0.82)
\\
 \hline
\multirow{2}*{{Balance}}	&86.40	&86.40	&86.56	&87.52	&87.52	&87.67	&86.75	&86.74	&\textbf{87.55}	&77.49	&77.49	&77.81	&76.70	&76.70	&76.70	&\textbf{88.97}
\\
 &(0.00)	&(0.00)	&(0.00)	&(0.00)	&(0.00)	&(0.00)	&(1.35)	&(1.96)	&(1.53)	&(0.00)	&(0.00)	&(0.00)	&(0.00)	&(0.00)	&(0.00)	&(0.40)
\\
 \hline
\multirow{2}*{{Yeast}}	&53.93	&53.93	&53.99	&39.24	&39.24	&54.06	&45.48	&43.82	&45.50	&44.96	&43.61	&50.53	&45.51	&46.94	&50.53	&\textbf{56.28}
 \\
&(0.00)	&(0.00)	&(0.00)	&(0.00)	&(0.00)	&(0.00)	&(0.96)	&(1.99)	&(1.51)	&(1.10)	&(0.93)	&(0.81)	&(1.65)	&(2.15)	&(0.99)	&(0.18)
 \\
 \hline
\multirow{2}*{{Satimage}}	&86.84	&86.84	&\textbf{86.92}	&82.36	&82.36	&82.29	&84.70	&84.47	&85.01	&83.26	&83.25	&83.26	&77.69	&79.08	&84.83	&85.74
\\
 &(0.00)	&(0.00)	&(0.00)	&(0.00)	&(0.00)	&(0.00)	&(0.55)	&(0.90)	&(0.34)	&(0.39)	&(0.24)	&(0.21)	&(2.77)	&(3.47)	&(0.61)	&(0.11)
\\
 \hline
\multirow{2}*{{Pendigits}}	&97.16	&97.16	&\textbf{97.24}	&84.88	&84.88	&86.25	&76.46	&76.05	&77.65	&86.13	&86.08	&87.13	&78.37	&77.87	&78.84	&96.70
\\
 &(0.00)	&(0.00)	&(0.00)	&(0.00)	&(0.00)	&(0.00)	&(1.03)	&(0.90)	&(1.18)	&(0.24)	&(0.44)	&(0.19)	&(1.00)	&(1.00)	&(1.28)	&(0.15)
\\
 \hline
\multirow{2}*{{Segmentation}}	&95.18	&95.18	&95.31	&90.03	&90.03	&93.06	&91.48	&91.28	&92.43	&92.46	&92.46	&94.20	&93.37	&93.37	&93.37	&\textbf{95.60}
\\
&(0.00)	&(0.00)	&(0.00)	&(0.00)	&(0.00)	&(0.00)	&(1.02)	&(1.02)	&(0.69)	&(0.00)	&(0.00)	&(0.00)	&(0.00)	&(0.00)	&(0.00)	&(0.18)
 \\
 \hline
\multirow{2}*{{OptDigts}}	&95.03	&95.03	&95.28	&83.27	&83.27	&84.09	&71.66	&72.80	&74.69	&85.80	&85.80	&86.03	&75.27	&75.27	&75.27	&\textbf{95.67}
\\
&(0.00)	&(0.00)	&(0.00)	&(0.00)	&(0.00)	&(0.00)	&(1.17)	&(2.06)	&(0.94)	&(0.00)	&(0.00)	&(0.00)	&(0.00)	&(0.00)	&(0.00)	&(0.13)
 \\
 \hline
\multirow{2}*{{Vehicle}}	&73.40	&73.40	&73.52	&65.12	&65.12	&72.33	&70.39	&70.21	&73.07	&68.16	&67.72	&72.35	&70.88	&71.29	&\textbf{74.28}	&\textbf{75.41}
\\
&(0.00)	&(0.00)	&(0.00)	&(0.00)	&(0.00)	&(0.00)	&(1.00)	&(1.30)	&(0.69)	&(0.81)	&(0.27)	&(0.32)	&(1.31)	&(1.02)	&(1.04)	&(0.13)
 \\
 \hline
 \hline
{\textbf{Rank}}&3.93 	&4.29 	&3.64 	&9.86 	&10.07 	&6.43 	&8.79 	&9.50 	&6.86 	&8.50 	&9.07 	&6.86 	&8.93 	&9.14 	&6.86 	&2.14
\\
 \hline
\end{tabular}}}
\end{table*}

 \begin{table*} [t]
\caption{\label{tab:acc2} {Accuracy (\%) comparison of the ECOC coding-decoding methods on the UCI datasets. The base learner is the {Gaussian RBF kernel based SVM}. In each grid, the first line is the accuracy and the second line is the standard deviation.
The row ``\textbf{Rank}'' is the average rank over all 14 datasets.}}
%\vspace{2mm}
\centerline{
\scalebox{0.86}{
%\begin{tabular*}{170mm}{|@{}l@{\extracolsep{\fill}}|c|c|c|c|c|c|c|c|c|c|c|c|c|c@{\extracolsep{\fill}}|} %{l*{14}{|c}}
\begin{tabular}{l||ccc|ccc|ccc|ccc|ccc|c}
 \hline
\textbf{Coding}	&\multicolumn{3}{c|}{1vs1}	&\multicolumn{3}{c|}{1vsALL	}&\multicolumn{3}{c|}{Random}	 &\multicolumn{3}{c|}{ECOC-ONE}	&\multicolumn{3}{c|}{DECOC}	&{WOLC-ECOC} \\
\hline
\textbf{Decoding}& HD & LB & LW & HD & LB & LW& HD & LB & LW& HD & LB & LW& HD & LB & LW& OW\\
\hline
\hline
 \multirow{2}*{{Dermathology}}	&\textbf{96.93}	&\textbf{96.76}	&\textbf{96.88}	&94.87	&94.63	&95.82	&94.82	&95.30	&95.94	&94.73	&94.74	&95.59	&94.76	&95.16	&95.40	&95.17
  \\
&(0.59)	&(0.35)	&(0.51)	&(0.46)	&(0.53)	&(0.84)	&(1.33)	&(1.00)	&(0.49)	&(2.88)	&(2.51)	&(0.55)	&(0.71)	&(0.78)	&(0.86)	&(0.55)
 \\
 \hline
 \multirow{2}*{{Iris}}	&\textbf{96.80} &96.66	&96.51	&95.41	&95.00	&\textbf{96.67}	&\textbf{97.52}	&\textbf{97.30}	&\textbf{96.91}	&96.32	&96.19	&\textbf{96.87}	&\textbf{96.97}	&\textbf{96.86}	&96.75	&96.69
 \\
 &(0.96)	&(0.63)	&(0.60)	&(1.54)	&(0.72)	&(1.08)	&(0.76)	&(0.38)	&(0.67)	&(0.47)	&(1.33)	&(0.76)	&(0.57)	&(0.79)	&(0.79)	&(0.37)
\\
 \hline
\multirow{2}*{{Ecoli}}	&85.07	&\textbf{85.17}	&84.81	&80.52	&80.72	&82.75	&80.93	&81.09	&82.40	&81.59	&81.66	&83.28	&74.59	&74.39	&82.70	&83.49
\\
 &(0.81)	&(0.60)	&(0.75)	&(1.03)	&(0.79)	&(0.98)	&(2.15)	&(2.12)	&(1.13)	&(1.13)	&(0.73)	&(0.67)	&(5.18)	&(6.04)	&(1.40)	&(0.25)
\\
 \hline
\multirow{2}*{{Wine}}	&96.05	&96.16	&96.33	&\textbf{96.65}	&96.15	&\textbf{96.60}	&\textbf{97.37}	&\textbf{96.93}	&\textbf{97.04}	&\textbf{97.16}	&96.64	&\textbf{96.70}	&96.38	&\textbf{96.77}	&\textbf{96.60}	&95.85
 \\
&(1.20)	&(0.79)	&(0.85)	&(0.87)	&(0.78)	&(0.89)	&(0.76)	&(0.93)	&(0.58)	&(0.81)	&(0.63)	&(0.70)	&(0.96)	&(0.67)	&(0.99)	&(0.80)
 \\
 \hline
\multirow{2}*{{Glass}}	&\textbf{62.95}	&\textbf{63.84}	&\textbf{64.01}	&52.98	&52.03	&61.27	&61.00	&\textbf{62.09}	&\textbf{61.57}	&56.59	&56.60	&\textbf{63.06}	&58.10	&56.75	&59.91	&\textbf{63.18}
 \\
 &(1.79)	&(2.01)	&(3.16)	&(2.61)	&(1.99)	&(1.37)	&(2.24)	&(2.49)	&(2.03)	&(2.03)	&(2.22)	&(2.54)	&(5.02)	&(4.08)	&(2.76)	&(1.80)
\\
 \hline
\multirow{2}*{{Thyroid}}	&\textbf{96.20}	&\textbf{96.14}	&\textbf{96.22}	&95.21	&\textbf{95.45}	&\textbf{95.93}	&\textbf{96.21}	&\textbf{96.23}	&\textbf{95.77}	&94.99	&94.64	&\textbf{95.69}	&94.50	&94.51	&93.67	&95.63
 \\
 &(0.75)	&(0.60)	&(0.81)	&(1.03)	&(0.96)	&(0.62)	&(0.93)	&(0.69)	&(0.67)	&(0.83)	&(1.05)	&(0.63)	&(1.37)	&(1.27)	&(1.02)	&(0.56)
\\
 \hline
\multirow{2}*{{Vowel}}	&67.11	&67.81	&67.87	&34.88	&34.67	&36.96	&31.07	&33.17	&34.37	&37.56	&37.55	&39.87	&43.40	&41.04	&41.87	&\textbf{70.87}
 \\
 &(1.40)	&(1.19)	&(1.84)	&(0.78)	&(1.58)	&(1.36)	&(2.89)	&(1.88)	&(2.05)	&(2.43)	&(1.69)	&(1.47)	&(3.33)	&(2.54)	&(1.62)	&(1.20)
\\
 \hline
\multirow{2}*{{Balance}}	&90.12	&88.89	&89.48	&90.25	&\textbf{90.34}	&90.28	&89.74	&89.78	&89.66	&88.19	&87.71	&87.35	&88.76	&88.95	&88.74	&\textbf{91.29}
\\
&(1.07)	&(1.41)	&(0.99)	&(0.88)	&(1.28)	&(0.93)	&(0.70)	&(0.94)	&(0.98)	&(0.49)	&(0.75)	&(1.55)	&(0.91)	&(0.65)	&(0.61)	&(0.92)
 \\
 \hline
\multirow{2}*{{Yeast}}	&\textbf{58.98}	&\textbf{58.95}	&\textbf{59.35}	&38.17	&38.41	&54.73	&51.90	&50.97	&53.19	&43.00	&43.71	&54.98	&51.97	&51.97	&55.14	&55.27
 \\
&(1.10)	&(0.56)	&(0.63)	&(1.38)	&(1.44)	&(0.62)	&(1.02)	&(2.22)	&(1.35)	&(2.26)	&(2.24)	&(1.02)	&(2.35)	&(3.11)	&(1.33)	&(0.57)
 \\
 \hline
\multirow{2}*{{Satimage}}	&85.73	&\textbf{85.74}	&\textbf{85.81}	&80.07	&79.98	&81.05	&81.95	&81.43	&82.07	&81.20	&81.27	&81.49	&74.95	&75.86	&82.48	&\textbf{86.10}
\\
 &(0.19)	&(0.21)	&(0.20)	&(0.18)	&(0.30)	&(0.27)	&(0.45)	&(0.90)	&(0.58)	&(0.62)	&(0.69)	&(0.81)	&(3.47)	&(3.20)	&(0.68)	&(0.27)
\\
 \hline
\multirow{2}*{{Pendigits}}	&\textbf{99.01}	&\textbf{99.01}	&\textbf{98.97}	&91.79	&91.69	&92.29	&85.19	&85.20	&86.05	&92.53	&92.60	&93.24	&88.23	&88.43	&88.97	&98.25
\\
&(0.06)	&(0.06)	&(0.06)	&(0.19)	&(0.15)	&(0.17)	&(1.16)	&(0.76)	&(0.74)	&(0.23)	&(0.16)	&(0.31)	&(1.50)	&(1.15)	&(0.83)	&(0.16)
 \\
 \hline
\multirow{2}*{{Segmentation}}	&\textbf{94.86}	&\textbf{95.14}	&\textbf{95.06}	&85.20	&85.16	&89.45	&86.41	&86.93	&87.60	&89.21	&89.23	&91.79	&87.03	&86.93	&86.67	&\textbf{95.12}
\\
 &(0.45)	&(0.47)	&(0.42)	&(0.98)	&(0.76)	&(0.70)	&(1.54)	&(1.72)	&(1.16)	&(0.78)	&(0.66)	&(0.58)	&(0.89)	&(1.08)	&(1.08)	&(0.30)
\\
 \hline
\multirow{2}*{{OptDigts}}	&\textbf{97.80}	&\textbf{97.74}	&\textbf{97.79}	&92.99	&92.88	&94.39	&88.42	&88.18	&89.35	&94.66	&94.74	&94.79	&89.33	&89.33	&89.23	&97.58
\\
&(0.09)	&(0.12)	&(0.07)	&(0.13)	&(0.14)	&(0.15)	&(0.71)	&(1.37)	&(0.81)	&(0.16)	&(0.13)	&(0.18)	&(0.23)	&(0.31)	&(0.17)	&(0.11)
 \\
 \hline
\multirow{2}*{{Vehicle}}	&79.62	&79.66	&80.01	&69.02	&68.53	&75.49	&75.16	&76.28	&77.77	&71.34	&72.12	&76.85	&74.48	&75.12	&76.99	&\textbf{82.51}
\\
&(0.93)	&(0.83)	&(0.64)	&(0.70)	&(1.43)	&(0.82)	&(0.60)	&(1.62)	&(1.32)	&(1.03)	&(1.06)	&(1.50)	&(0.83)	&(1.74)	&(1.33)	&(0.34)
 \\
 \hline
 \hline
 {\textbf{Rank}}	&2.07 	&2.79 	&2.43 	&10.64 	&10.86 	&6.14 	&8.29 	&6.93 	&6.07 	&8.36 	&8.79 	&5.07 	&9.71 	&9.21 	&7.29 	&4.57
\\
 \hline
\end{tabular}}}
\end{table*}

The reason why the WOLC-ECOC with AdaBoost performs better than the WOLC-ECOC with SVM may be explained from information theory. It is well known in information theory that the error-correcting ability of any coding method is upper-bounded by the \textit{Shannon limit} which is irrelevant to the coding method. That is to say, it is possible that the performance of a strong coding method in a noisy channel is worse than the performance of a weak coding method in a clean channel.

 The channel of an ECOC problem, as presented in Section \ref{sec:mmc}, is determined by the features, base learner and coding method. (i) The more suitable the bipartitions of the classes are and the stronger the base learner is, the cleaner the channel will be. Because 1vs1 bipartitions data according to their natural distributions, its channel has minimum noise in most datasets. Similarly, AdaBoost introduces more noise to the channel than SVM. We can image that the {Shannon limits} of different coding methods with AdaBoost as the base learner tend to be more similar than those with SVM as the base learner. (ii) On the other side, the more diverse the dichotomizers are and the larger the minimum distance between the codewords is, the stronger the error-correcting ability of the codes will be, where the term ``diverse'' is also named \textit{independent} in some papers \cite{bagheri2012subspace,bagheri2012rough}.

 When the Shannon limits are similar, the performance is determined by the error-correcting ability of the coding methods, which explains the advantage of WOLC-ECOC in Table \ref{tab:acc1}; otherwise, the performance is determined by the Shannon limits, which explains the inferior of the WOLC-ECOC to 1vs1 coding in Table \ref{tab:acc2}.

Note that WOLC-ECOC was initialized by 1vsALL in all experiments. If it is initialized by other coding methods that are better than 1vsALL, it may achieve better performance.

\subsection{Efficiency}
The efficiency of an ECOC method is influenced by its code length. The shorter the code length is, the more efficient the ECOC method will be.

Table \ref{tab:efficient} lists the code lengths of all comparison methods. From the table, we can see that WOLC-ECOC has a much shorter code length than 1vs1, though it has a slightly longer code length than the other codings. Generally, it is worthy sacrificing some efficiency for much better performance.

\begin{table*} [t]
\caption{\label{tab:efficient} {Code length comparison of the ECOC coding-decoding methods on the UCI datasets.}}
%\vspace{2mm}
\centerline{
\scalebox{1}{
%\begin{tabular*}{170mm}{|@{}l@{\extracolsep{\fill}}|c|c|c|c|c|c|c|c|c|c|c|c|c|c@{\extracolsep{\fill}}|} %{l*{14}{|c}}
\begin{tabular}{l||c|c|c|cc|cc|cc|c|cc}
\hline
\textbf{Coding}	&{1vs1}	&{1vsALL}	&{Random}	 &\multicolumn{6}{c|}{ECOC-ONE}	&{DECOC}	&\multicolumn{2}{c}{WOLC-ECOC} \\
\hline
\textbf{Decoding}&{--}	&{--}	&{--}	 &\multicolumn{2}{c|}{HD}&\multicolumn{2}{c|}{LB}&\multicolumn{2}{c|}{LW}	&{--}	&\multicolumn{2}{c}{OW} \\
\hline
 \textbf{Base classifier} &-- & --& --&Ada &SVM &Ada &SVM &Ada &SVM &-- & Ada& SVM\\
 \hline
 \hline
 \multirow{2}*{{Dermathology}}	&15.00&	6.00 & 10.00 &7.09 &7.26 & 7.11&7.30 &7.50 &7.74 &5.00	&9.09 &   6.00\\
 & (0.00) & (0.00) & (0.00) &(0.08) 	&(0.37) 	&(0.09) 	&(0.28) 	&(0.00) 	&(0.42) & (0.00) & (1.27) & (0.00)
 \\
 \hline
\multirow{2}*{{Iris}}	&3.00 &	3.00 &	 10.00 &4.50 &4.93 &4.50 & 4.99& 6.31&6.68 &  2	.00& 7.00&  5.93\\
& (0.00) & (0.00) & (0.00)&(0.00) 	&(0.81) 	&(0.00) 	&(0.55) 	&(0.26) 	&(0.39) & (0.00)& (0.00) & (0.78)
 \\
 \hline
\multirow{2}*{{Ecoli}}	&28.00 &	8.00 &	 10.00 &9.48 &9.05 &9.46 & 9.10&9.65 &9.29 &  7.00	& 14.75&  15.24\\
& (0.00) & (0.00) & (0.00)&(0.13) 	&(0.09) 	&(0.13) 	&(0.13) 	&(0.20) 	&(0.20) & (0.00)& (2.01) & (4.16)
 \\
 \hline
\multirow{2}*{{Wine}}	&3.00 &	3.00 &	 10.00 &7.00 &7.00 &7.00 &7.36 &7.00 &7.36 &  2.00 & 3.00&  3.00\\
& (0.00) & (0.00) & (0.00)&(0.00) 	&(0.45) 	&(0.00) 	&(0.36) 	&(0.00) 	&(0.38) & (0.00)& (0.00) & (0.00)
 \\
 \hline
 \multirow{2}*{{Glass}}	&15.00 &	6.00 &	 10.00 &7.35 &7.40 &7.23 &7.39 &7.93 &7.61 &  5.00 & 9.44& 12.50\\
 & (0.00) & (0.00) & (0.00)&(0.13) 	&(0.15) 	&(0.11) 	&(0.18) 	&(0.41) 	&(0.32) & (0.00)& (0.37) & (1.08)
\\
 \hline
\multirow{2}*{{Thyroid}}	&3.00 &	3.00 &	 10.00&6.63 &6.33 &6.63 &6.45 & 6.63 &6.18 &  2.00 & 3.00&   3.35\\
& (0.00) & (0.00) & (0.00)&(0.00) 	&(0.32) 	&(0.00) 	&(0.74) 	&(0.00) 	&(0.56) & (0.00)& (0.00) & (0.26)
 \\
 \hline
\multirow{2}*{{Vowel}}	&55.00 &	11.00 &	 10.00 & 12.10&12.20 & 12.05&12.23 &12.10 &12.05 &   10.00& 26.64&  24.25\\
 & (0.00) & (0.00) & (0.00)&(0.11) 	&(0.13) 	&(0.06) 	&(0.15) 	&(0.11) 	&(0.06) & (0.00)& (0.58) & (2.59)
\\
 \hline
 \multirow{2}*{{Balance}}	&3.00 &	3.00 &	 10.00 &8.00 &7.93 &8.00 &7.69 & 8.00 & 7.71&  2.00  & 15.16&  13.60\\
& (0.00) & (0.00) & (0.00)&(0.00) 	&(0.16) 	&(0.00) 	&(0.39) 	&(0.00) 	&(0.45) & (0.00)& (1.96) & (3.29)
 \\
 \hline
\multirow{2}*{{Yeast}}	&45.00 &	10.00 &	 10.00 &11.20 &11.13 &11.14 &11.11 & 12.73 & 11.19&  9.00  & 13.30&   16.45\\
& (0.00) & (0.00) & (0.00)&(0.15) 	&(0.10) 	&(0.12) 	&(0.09) 	&(0.29) 	&(0.24) & (0.00)& (0.63) & (2.55)
 \\
 \hline
\multirow{2}*{{Satimage}}	&15.00 &	6.00 &	 10.00 & 7.09&7.06 &7.04 &7.10 &7.00 & 7.60&  5.00  & 10.70&  16.78\\
& (0.00) & (0.00) & (0.00)&(0.10) 	&(0.07) 	&(0.08) 	&(0.13) 	&(0.00) 	&(0.32) & (0.00)& (2.52) & (5.18)
 \\
 \hline
\multirow{2}*{{Pendigits}}	&45.00&	10.00 &	 10.00 &11.43 & 11.10&11.38 & 11.13&11.04 & 11.09&  9.00 & 24.06&   22.74\\
& (0.00) & (0.00) & (0.00)&(0.17) 	&(0.11) 	&(0.18) 	&(0.17) 	&(0.06) 	&(0.12) & (0.00)& (4.43) & (6.23)
 \\
 \hline
\multirow{2}*{{Segmentation}}	&21.00 &	7.00 &	 10.00&8.00 &8.00 &8.00 &8.00 & 8.25&8.08 & 6.00  & 13.18 &  14.71\\
 & (0.00) & (0.00) & (0.00)&(0.00) 	&(0.00) 	&(0.00) 	&(0.00) 	&(0.00) 	&(0.09) & (0.00)& (2.05) & (3.30)
\\
 \hline
\multirow{2}*{{OptDigts}}	&45.00 &	10.00 &	 10.00 &11.00 & 11.00&11.00 &11.00 &11.00 &11.08 &  9.00 & 22.45 &   24.14 \\
 & (0.00) & (0.00) & (0.00)&(0.00) 	&(0.00) 	&(0.00) 	&(0.00) 	&(0.00) 	&(0.12) & (0.00)& (5.62) & (1.93)
\\
 \hline
\multirow{2}*{{Vehicle}}	&6.00 &	4.00 &	 10.00 &5.05 &5.09&5.02 &5.00 & 5.43 & 5.68&  2.00 & 10.96&  13.56\\
& (0.00) & (0.00) & (0.00)&(0.07) 	&(0.12) 	&(0.05) 	&(0.00) 	&(0.43) 	&(0.46) & (0.00)& (0.68) & (4.60)
 \\
 \hline
\end{tabular}}}
\end{table*}

\subsection{Study of the Convergence Behavior}
In this subsection, we verify the convergence behavior of WOLC-ECOC empirically. For simplicity, we only give two examples on the {Dermathology} and {Vehicle} datasets, which are shown in Figs. \ref{fig:fig4} and \ref{fig:fig44} respectively.
 The training risk (i.e., objective value) in both figures is calculated by (\ref{eq:7}), and the accuracy is defined as the ratio of the number of correctly classified training/test examples over the total number.

  \begin{figure}[!t]
\centerline{
\includegraphics[width=90mm]{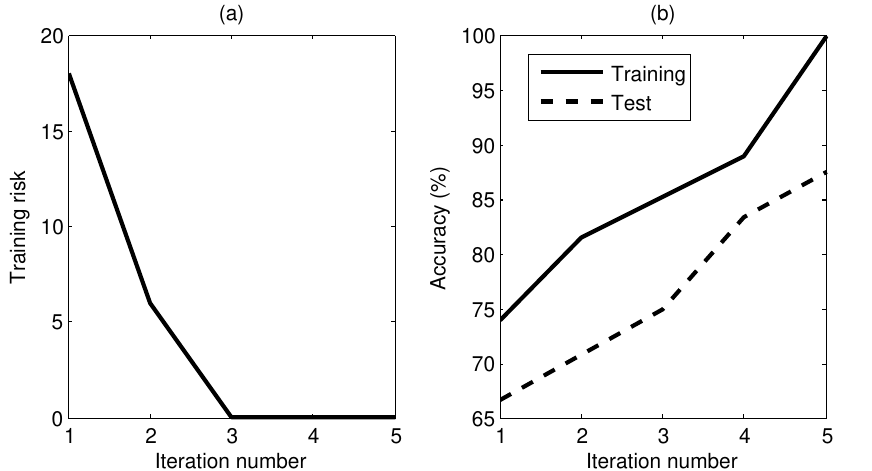}}
\caption{{Convergence behavior of WOLC-ECOC on the Dermathology dataset with discrete AdaBoost as the base learner. (a) Convergence behavior of the training risk (objective value). (b) Curves of the training and test accuracies.}}
\label{fig:fig4}
\end{figure}

 \begin{figure}[!t]
\centerline{
\includegraphics[width=90mm]{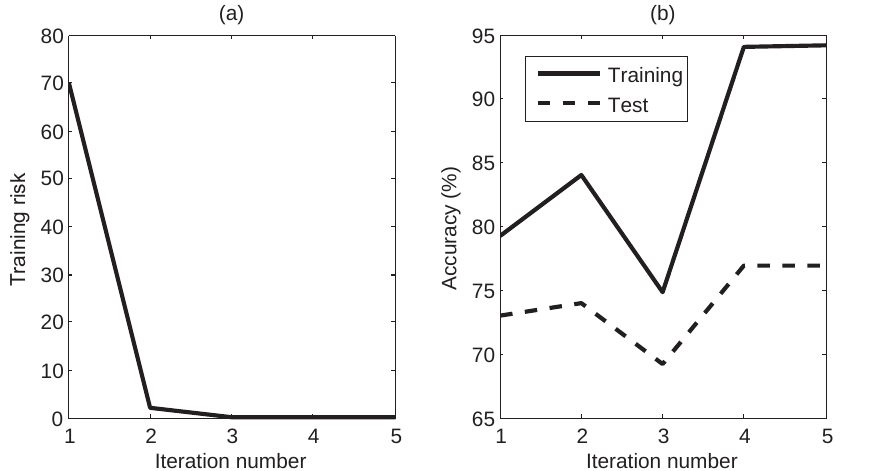}}
\caption{{Convergence behavior of WOLC-ECOC on the Vehicle dataset with discrete AdaBoost as the base learner. (a) Convergence behavior of the training risk (objective value). (b) Curves of the training and test accuracies.}}
\label{fig:fig44}
\end{figure}

 From the figures, we observe that the training risks decrease rigorously with respect to the numbers of training iterations. We also observe that the training and test accuracies increase in general along with the decrease of the objective values.

\subsection{Application to Music Genre Classification}
The fast development of multimedia technologies enable people to enjoy a large amount of music, which calls for developing tools to classify music effectively and efficiently. The SVM based 1vs1 and 1vsALL classifier ensembles are popular for the music classification problems \cite{li2006toward}.
The purpose of this subsection is to show the advantage of the WOLC-ECOC over the aforementioned two coding methods on this problem.

The music genre dataset is the {Dortmund} dataset \cite{homburg2005benchmark}\footnote{\url{http://www-ai.cs.uni-dortmund.de/audio.html}}. It consists of 1886 recordings of music pieces of 10-seconds duration, which are classified to 9 types of music. Each music piece is a 44.1kHz, 16-bits, stereo MP3 file.
Here, we converted each file to a mono audio file and extracted three kinds of acoustic features from the file as in \cite{lee2009automatic}, which were the Modulation spectral analysis of the Mel-Frequency Cepstral Coefficients (MMFCC), Octave-based Spectral Contrast (MOSC), and Normalized Audio Spectral Envelope (MNASE). As a result, each file was formulated as an example with 3 kinds of features. The parameters settings of the ECOC methods and SVM were as same as those in Section \ref{subsec:experimental_settings}.

Tables \ref{tab:musicCompare1} and \ref{tab:musicCompare2} list the accuracy and code length comparisons of the ECOCs with the 3 acoustic features. From Table \ref{tab:musicCompare1}, it is clear that WOLC-ECOC is the most powerful one. From Table \ref{tab:musicCompare2}, we observe that the code length of WOLC-ECOC is much shorter than 1vs1, though the code length of WOLC-ECOC is slightly longer than the other three methods.

\begin{table} [t]
\caption{\label{tab:musicCompare1} {Accuracy (\%) comparison of the ECOC coding-decoding methods on the Dortmund music dataset with 3 kinds of features. In each grid, the first line is the accuracy and the second line represents its corresponding decoding method.
}}
%\vspace{2mm}
\centerline{
\scalebox{0.93}{
\begin{tabular}{l||c|c|c|c|c}
 \hline
\textbf{Coding} & 1vs1& 1vsALL& DECOC&	ECOC-ONE& WOLC-ECOC\\
 \hline
 \hline
\multirow{2}*{MMFCC} & 43.15&	47.33&	45.34&	49.00&	\bf{50.49}\\
&LW &  LW& LW& LW& OW\\
\hline
\multirow{2}*{MOSC} & 44.41&	47.89&	46.76&	50.15&	\bf{52.78}\\
& LW &  LW& LW& LW& OW\\
 \hline
\multirow{2}*{MNASE}& 45.75&	50.85&	46.42&	50.93&	\bf{52.86}\\
&LW &  LW& LW& LW& OW\\
 \hline
\end{tabular}}}
\end{table}

 \begin{table} [t]
\caption{\label{tab:musicCompare2} {Code length comparison of the ECOC coding-decoding
 methods on the Dortmund dataset with 3 kinds of features.}}
%\vspace{2mm}
\centerline{
\scalebox{0.93}{
\begin{tabular}{l||c|c|c|c|c}
 \hline
\textbf{Coding} & 1vs1& 1vsALL&	DECOC& ECOC-ONE& WOLC-ECOC\\
 \hline
 \hline
MMFCC & 45.00&	9.00&	8.00&	16.62&	27.64\\
\hline
MOSC & 45.00&	9.00&	8.00&	14.24&	22.78\\
 \hline
MNASE& 45.00&	9.00&	8.00&	14.75&	24.23\\
 \hline
\end{tabular}}}
\end{table}

Figure \ref{fig:musicfig1} gives an example of the convergence behavior of the training risk of WOLC-ECOC with MNASE as the feature. From Fig. \ref{fig:musicfig1}a, we observe that the training risk decreases rigorously with respect to the number of iterations.

 \begin{figure}[!t]
\centerline{
\includegraphics[width=90mm]{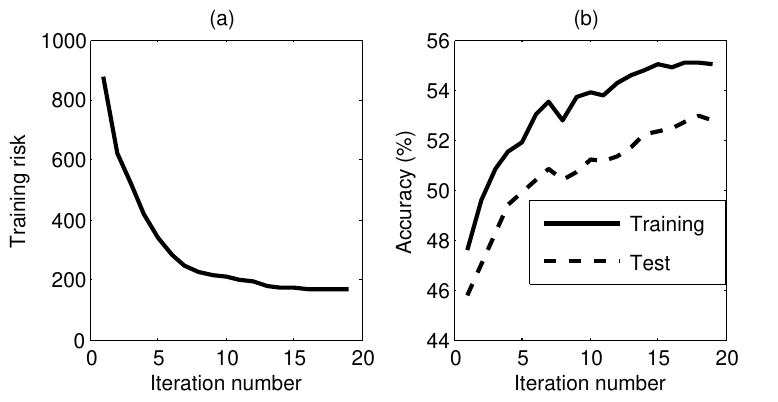}}
\caption{Convergence behavior of WOLC-ECOC on the Dortmund music genre dataset with MNASE as the acoustic feature.}
\label{fig:musicfig1}
\end{figure}

\section{Conclusions}\label{sec:conclusion}

In this paper, we have proposed a heuristic ternary WOLC-ECOC. First, we have proposed LC-ECOC, a greedy training method that iteratively constructs multiple strong dichotomizers to discriminate the most confusing binary-class problem.
 Then, we have proposed the CPA based OW decoding. OW decoding improves LW decoding by optimizing the weight matrix of the latter for the minimum training risk. The optimization problem is further solved by CPA, which makes the OW decoding have linear time and storage complexities.
  At last, we have proposed WOLC-ECOC, which iteratively executes LC-ECOC and the CPA based OW decoding until the training risk converges. WOLC-ECOC not only inherits all merits of LC-ECOC and the CPA based OW decoding but also ensures the decrease of the training risk.

We have conducted an extensive experimental comparison with 15 state-of-the-art ECOC coding-decoding pairs on 14 UCI datasets with the discrete AdaBoost and well-tuned RBF kernel based SVM as two base learners.
 Experimental results have shown that (i) when Adaboost is used as the base learner, WOLC-ECOC outperforms all 15 coding-decoding pairs; (ii) when SVM is used as the base learner, WOLC-ECOC is weaker than the traditional 1vs1 coding method but better than other coding methods; (iii) the code length of WOLC-ECOC is much shorter than that of 1vs1. We have explained the experimental phenomena in the view of information theory.
 Moreover, we have applied WOLC-ECOC to a music genre classification problem. Experimental results have shown that WOLC-ECOC outperforms all referenced coding methods including 1vs1.

 %In the future, we will try to analyze the ``channel capacity'' of ECOC theoretically so as to guide our ECOC design. We will also try to develop better bipartition techniques on the classes for cleaner ``channels'' and try to incorporate other types of diversity enhancement techniques so as to make the error-correcting ability of ECOC more apparent.

\ifCLASSOPTIONcaptionsoff
  \newpage
\fi

\appendix

\subsection{Proof of Theorem \ref{thm:1}}\label{app:proof_thm1}
The proof is similar with the proof of \cite[Theorem 1]{joachims2006training}.
The key point is to prove that the training loss of problem (\ref{eq:9}) and the training loss of problem (\ref{eq:8}) are equivalent:
 \begin{eqnarray}
\sum_{i=1}^{n}\xi_i&=&\sum_{i=1}^{n}\max_{p=1,\ldots,P}\left(0,\w_{y_i}^T\u_{i,y_i}-\w_{p}^T\u_{i,p}\right)\nonumber\\
&=&\sum_{i=1}^{n}\max_{\forall\g_i\in\mathcal{Z}}\left(\sum_{p=1}^{P}g_{i,p}\left( \w_{y_i}^T\u_{i,y_i}-\w_{p}^T\u_{i,p} \right) \right)
\label{eq:app1}
 \end{eqnarray}
 where set $\mathcal{Z} = \left\{\z_1,\ldots,\z_P \right\}$ with $\z_p$ defined as
 \begin{eqnarray}
z_{p,k} = \left\{\begin{array}{ll}1&,\mbox{ if } k=p\\0&,\mbox{ otherwise }\end{array}\right. ,\quad k = 1,\ldots,P.
 \end{eqnarray}
Equation (\ref{eq:app1}) can be reformulated as
 \begin{eqnarray}
\sum_{i=1}^{n}\xi_i&=&\max_{\forall\G\in\mathcal{Z}^n}\left(\sum_{i=1}^{n}\sum_{p=1}^{P}g_{i,p}\left( \w_{y_i}^T\u_{i,y_i}-\w_{p}^T\u_{i,p} \right) \right) = \xi\nonumber\\
\label{eq:app2}
 \end{eqnarray}
 where $\G$ is defined as $\G = [\g_1,\ldots,\g_n]=\left[\renewcommand{\arraystretch}{0}\renewcommand{\arraycolsep}{0pt}\begin{array}{ccc}g_{1,1}&\ldots&g_{n,1}\\ \vdots&\ddots&\vdots\\g_{1,P}&\ldots&g_{n,P}\end{array}\right]$.
Theorem \ref{thm:1} is proved.

\subsection{Proof of the Monotonic Non-increase of the Training Risk of WOLC-ECOC}\label{app:proof_thm2}
Given the coding matrix $\M^{(t)}$, WOLC-ECOC classifier ensemble $\mathcal{C}^{(t)}$, minimum training risk $\J^{(t)}_{o}$, and optimal weight matrix ${\W^{(t)}_{o}}$ of the $t$-th iteration, where $\mathcal{C}^{(t)} = \{ h_1,h_2,\ldots,h_q\}$ with $q$ denoting the code length of the $t$-th iteration, and
 \begin{eqnarray}
\J^{(t)}_{o} &=& \min_{\W^{(t)}\in\mathcal{W}^{(t)}} \J^{(t)}\left(\W^{(t)}\right), \label{chaWOLC-ECOC:eq:1}\\
 {\W^{(t)}_{o}} &=& \arg\min_{\W^{(t)}\in\mathcal{W}^{(t)}} \J^{(t)}\left(\W^{(t)}\right) \label{chaWOLC-ECOC:eq:2}
 \end{eqnarray}
 with the training risk function $\J^{(t)}\left(\W^{(t)}\right)$ defined in (\ref{eq:7}). Suppose we get a new dichotomizer $h_{q+1}$ at the $(t+1)$-th iteration, we can obtain $\M^{(t+1)}$, $\mathcal{C}^{(t+1)}$, $\J^{(t+1)}_{o}$, and ${\W^{(t+1)}_{o}}$ in the same way as we did in the $t$-th iteration, where $\mathcal{C}^{(t+1)} = \mathcal{C}^{(t)}\cup h_{q+1}$ and $\M^{(t+1)} = [\M^{(t)},\mathbf{m}]$ with $\m$ denoted as the most difficult binary-class problem (in a vector form). We have the following theorem:
  \begin{thm}\label{thm:2}
   The non-increase of the training risk of WOLC-ECOC is guaranteed by the OW decoding:
 \begin{eqnarray}
 \J^{(0)}_{o}\ge\J^{(1)}_{o}\ge,\ldots,\ge \J^{(t)}_{o}\ge\J^{(t+1)}_{o}\ge,\ldots\nonumber
 \end{eqnarray}
 \end{thm}
 \begin{proof}
 We extend the optimal weight matrix $\W^{(t)}_o$ to another equivalent form ${\W^{(t+1)}}'=\left[\W^{(t)}_o, \mathbf{0}_{P\times 1}\right]$. It is easy to know that ${\W^{(t+1)}}'\in\mathcal{W}^{(t+1)}$. Because ${\W^{(t+1)}}'$ yields an objective value that is equivalent to $\J^{(t)}_{o}$, and also because ${\W^{(t+1)}}'$ is a point in $\mathcal{W}^{(t+1)}$ and problem (\ref{eq:7}) is a convex optimization problem with $\J^{(t+1)}_{o}$ as its minimum value, the inequality $ \J^{(t)}_{o}\ge\J^{(t+1)}_{o}$ holds. Theorem \ref{thm:2} is proved.
\end{proof}

  \section*{Acknowledgment}
 The author thanks the editors and the anonymous referees for their valuable advice. The author also thanks the researchers who opened the codes of their excellent works.

\bibliographystyle{IEEEtran}
\bibliography{zxlrefs2}

%\begin{IEEEbiography}[{\includegraphics[height=1.25in,clip,keepaspectratio]{zxlfig.pdf}}]{Xiao-Lei Zhang}
%(S'08-M'12) received the Ph.D. degree in Information and Communication Engineering from Tsinghua University, Beijing, China, in 2012. From 2013 to 2014, he was a visiting scholar with the Department of Computer Science and Engineering, Ohio State University, Columbus, OH, USA.
%His current research interests are the topics on machine learning, computational audition, bioinformatics, data mining, and information retrieval. He is a member of IEEE and IEEE Signal Processing Society.
%\end{IEEEbiography}

% that's all folks
\end{document}